\documentclass{article}
\usepackage[a4paper, total={6in, 10in}]{geometry}

\usepackage{mymath,natbib}
\usepackage[colorlinks=true,
linkcolor=blue,
anchorcolor=blue,
citecolor=blue
]{hyperref}

\title{Lifelong Learning with Sketched Structural Regularization}
\author{
  Haoran Li\thanks{Johns Hopkins University, Baltimore, MD 21218. \texttt{hli143@jhu.edu}.}
  \and
  Aditya Krishnan\thanks{Johns Hopkins University, Baltimore, MD 21218. \texttt{akrish23@jhu.edu}.}
  \and
  Jingfeng Wu\thanks{Johns Hopkins University, Baltimore, MD 21218. \texttt{uuujf@jhu.edu}.}
  \and
  Soheil Kolouri\thanks{HRL Laboratories, LLC, Malibu, CA 90265. \texttt{skolouri@hrl.com}.}
  \and
  Praveen K. Pilly\thanks{HRL Laboratories, LLC, Malibu, CA 90265. \texttt{pkpilly@hrl.com}.}
  \and
  Vladimir Braverman\thanks{Johns Hopkins University, Baltimore, MD 21218. \texttt{vova@cs.jhu.edu}.}
}
\date{}

\begin{document}

\maketitle

\begin{abstract}
    Preventing catastrophic forgetting while continually learning new tasks is an essential problem in lifelong learning. Structural regularization (SR) refers to a family of algorithms that mitigate catastrophic forgetting by penalizing the network for changing its ``critical parameters" from previous tasks while learning a new one. The penalty is often induced via a quadratic regularizer defined by an \emph{importance matrix}, e.g., the (empirical) Fisher information matrix in the Elastic Weight Consolidation framework. In practice and due to computational constraints, most SR methods crudely approximate the importance matrix by its diagonal. In this paper, we propose \emph{Sketched Structural Regularization} (Sketched SR) as an alternative approach to compress the importance matrices used for regularizing in SR methods. Specifically, we apply \emph{linear sketching methods} to better approximate the importance matrices in SR algorithms. We show that sketched SR: (i) is computationally efficient and straightforward to implement, (ii) provides an approximation error that is justified in theory, and (iii) is method oblivious by construction and can be adapted to any method that belongs to the structural regularization class. We show that our proposed approach consistently improves various SR algorithms' performance on both synthetic experiments and benchmark continual learning tasks, including permuted-MNIST and CIFAR-100.
\end{abstract}

\section{Introduction}
\emph{Lifelong learning}, also termed as \emph{continual learning} or \emph{incremental learning}, is the ability to continually learn in a varying environment through integrating the newly acquired knowledge while maintaining the previously learned experiences \citep{parisi2019continual}.
A central issue that prevents the state-of-the-art machine learning models (e.g. deep neural networks) from achieving lifelong learning is \emph{catastrophic forgetting}, i.e. learning a new task may severely modify the model parameters, including those that are important to the previous tasks \citep{parisi2019continual}.

\emph{Structural regularization} (SR), or \emph{selective synaptic plasticity}, is a general and widely-adopted paradigm to mitigate catastrophic forgetting in lifelong learning \citep{kolouri2019sliced,aljundi2018memory,kirkpatrick2017overcoming,chaudhry2018riemannian,zenke2017continual}.
From a geometric perspective~\citep{kolouri2019sliced,chaudhry2018riemannian}, SR methods construct an (positive semi-definite) \emph{importance matrix} (IM) that measures the relative importance of the model parameters to the old tasks (which are aimed be preserved in lifelong learning) and add a quadratic regularizer defined by the importance matrix when training on new tasks.
The intuition behind structural regularization is clear: the quadratic regularizer adaptively penalizes parameters from changing according to their criticality measured by the importance matrix. As a result, structural regularization encourages the model to fit to the new task using non-critical parameters so that it is able to preserve important information from old tasks.
For example, \citet{kirkpatrick2017overcoming} choose the (diagonal) \emph{empirical Fisher information matrix}\footnote{In their original paper \citep{kirkpatrick2017overcoming} (and follow-up papers, e.g., \citep{kolouri2019sliced}), the importance matrix in EWC is referred as the ``Fisher information matrix'', but precisely, it should be called the ``empirical Fisher'' --- the two terms are often used interchangeable in the community, though they are not identical. See \citep{kunstner2019limitations} for a detailed clarification.} (empirical Fisher, EF) as the importance matrix in their seminal algorithm, \emph{Elastic Weight Consolidation} (EWC) \citep{kirkpatrick2017overcoming,kolouri2019sliced,chaudhry2018riemannian}.
However, a full IM (e.g. empirical Fisher) scales as $\Ocal(m^2)$ for a model with $m$ parameters and can be prohibitively big to use in large-scale lifelong learning models. 
Often in practice, the diagonal, which scales as $\Ocal(m)$, is used as a crude approximation to the full IM \citep{kirkpatrick2017overcoming,kolouri2019sliced,aljundi2018memory}.
We refer to structural regularization with an IM approximated by the diagonal by \emph{diagonal SR}.

\begin{figure*}[t]
    \centering
    \begin{minipage}[c]{.5\textwidth}
        \centering
        \subfigure[Full EF]{\includegraphics[width=0.43\linewidth]{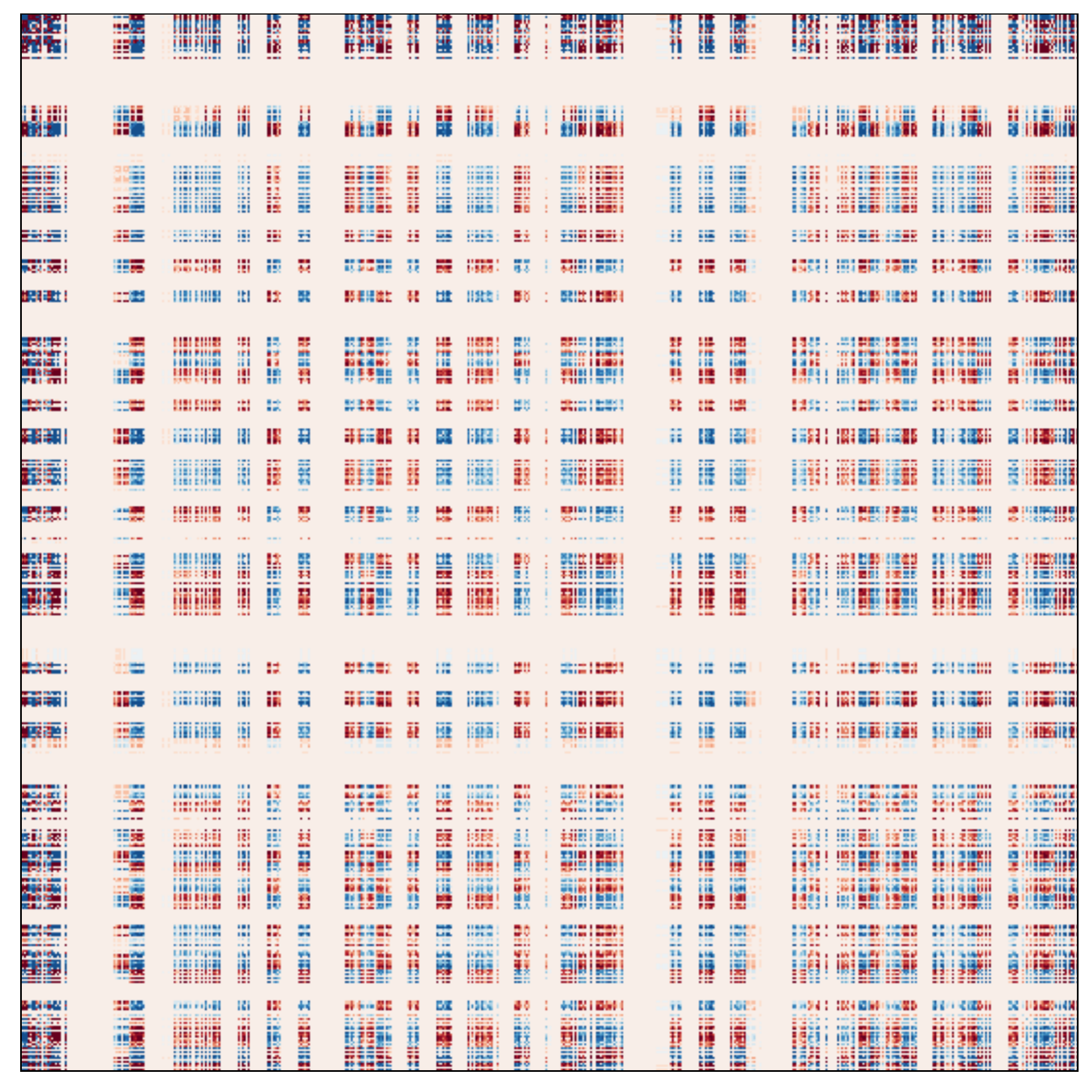}\label{fig:toy_EF}}
        \subfigure[Sketched EF]{\includegraphics[width=0.51\linewidth]{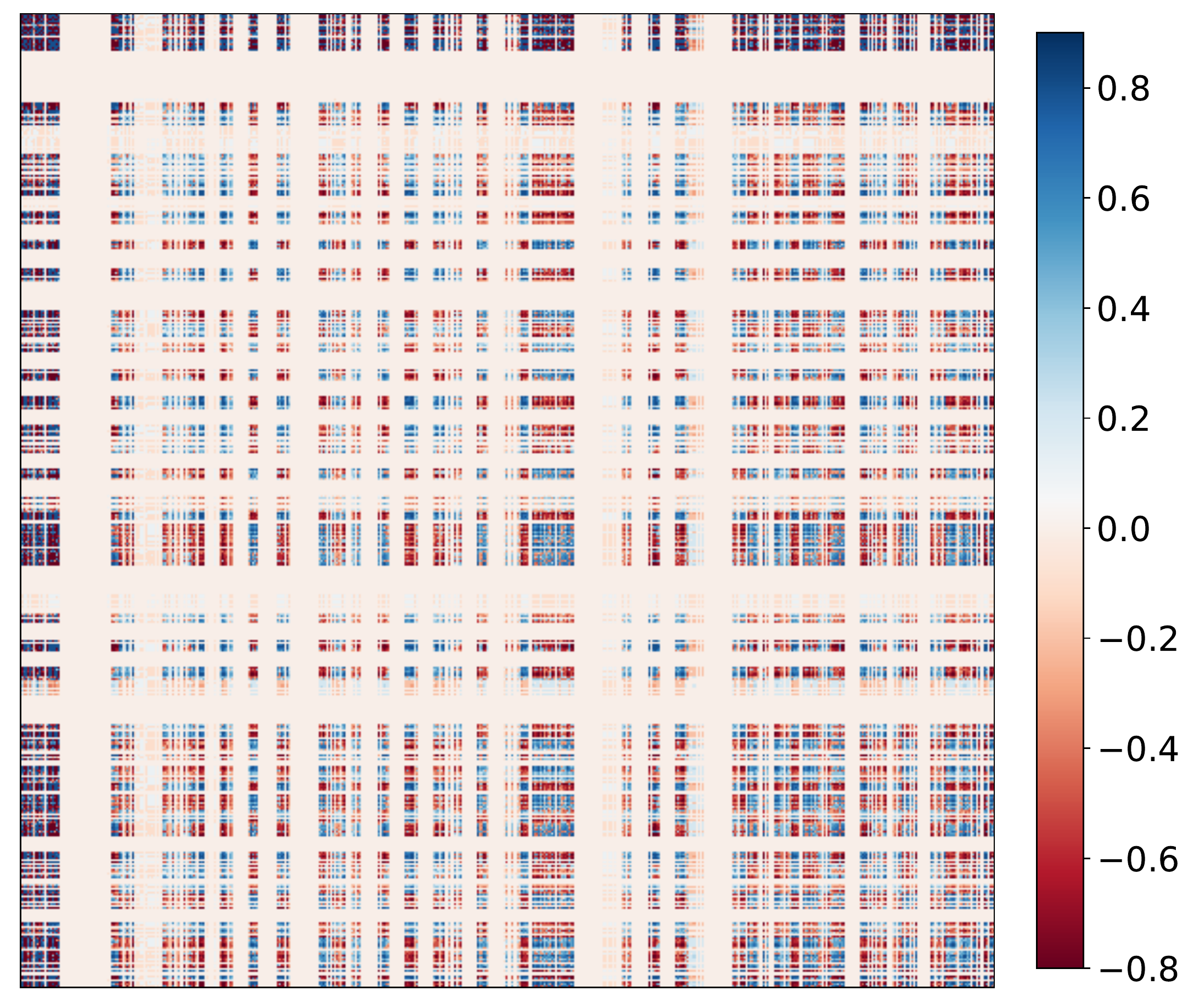}\label{fig:toy_SketchEF}}
    \end{minipage}%
    \begin{minipage}[c]{.5\textwidth}
        \centering
    \begin{subfigure}
        \centering
        \begin{small}
        \begin{sc}
        \begin{tabular}{cc}
            \toprule
            \multicolumn{2}{c}{Frobenius-Norm Approx. Error} \\ \multicolumn{2}{c}{${\|\widetilde\Omega - \Omega\|^2_F}/{\|\Omega \|^2_F}$} \\
            \midrule
            Diagonal & Block-Diagonal \\
            $ 94.7\% \pm 0.6\% $ & $ 82.8\% \pm 1.3\% $ \\
            \midrule
            Sketched & Low-Rank \\
            $\mathbf{ 8.1\% \pm 10.2\% }$ & $ 0.0\% \pm 0.0\% $ \\
            \bottomrule
        \end{tabular}
        \end{sc}
        \end{small}
    \end{subfigure}\label{table:toy_frobErr}
    \end{minipage}%
    
    \caption{Illustration of the (sketched) empirical Fisher on a synthetic 2D binary classification task from \citet{pan2020continual}. 
    The full empirical Fisher is obtained with the optimal weight that fits the first four tasks.
    Note that the empirical Fisher is normalized for better visualization; a true spectrum is shown in Figure \ref{fig:spectrum}.
    The figures show the heat map of the full empirical Fisher and the sketched empirical Fisher with $t=50$;
    the table shows the approximation error of each methods to approximating the full empirical Fisher.
    The plots and the table suggest that: (i) the full empirical Fisher cannot be well-approximated by its diagonal or block-diagonal, (ii) sketching method can utilize the off-diagonal entries to obtain a better approximation, and (iii) though low-rank method (with $50$ ranks) also leads to a good approximation, the computational cost is not affordable in practice.
    See Section \ref{sec:toy-exp} for more details.}
    \label{fig:fim_illustration}
\end{figure*}

While exploring new and effective importance matrices has been a hot direction for structural regularization \citep{kolouri2019sliced,aljundi2018memory,kirkpatrick2017overcoming,chaudhry2018riemannian,zenke2017continual}, little effort has been spent on examining the effectiveness of the crude diagonal approximation (a few exceptions, e.g. \citep{liu2018rotate,ritter2018online}, are discussed later in Section \ref{sec:relates}).
pIntuitively speaking, a diagonal IM assumes independence between parameters, which is far from reality \citep{liu2018rotate,ritter2018online}. In mathematics, a positive semi-definite matrix can rarely be well-approximated by its diagonal --- the only non-trivial exception to our knowledge is when the matrix is diagonally dominant \citep{horn2012matrix}. 
Unfortunately, for the importance matrices considered in SR methods this might not be the case, especially when training using neural networks. As an illustration, we examine the empirical Fisher as the importance matrix \citep{kirkpatrick2017overcoming} of a synthetic experiment from \citet{pan2020continual}; the full empirical Fisher is shown in Figure \ref{fig:toy_EF}. The plot shows that the full empirical Fisher is far from its diagonal; in fact the diagonal only contributes to less than $5.3\%$ of the Frobenius norm of the empirical Fisher matrix (see table in Figure \ref{fig:fim_illustration}).
Hence, approximating importance matrix with its diagonal might be problematic. A natural question then is: 

{\centering \emph{Is there a computational and memory efficient method to approximate the importance matrix without losing critical information in the matrix?}\par}

\paragraph{Our Contributions.}
In this paper, we answer the above question by providing a \emph{linear sketching method} \citep{charikar2002finding} as a \emph{provable}, \emph{ubiquitous}, \emph{efficient} and \emph{effective} approach to approximate the importance matrix in SR methods.
Specifically, in one pass of the data (which is also required for diagonal approximation), a $\Ocal(tm)$ size sketched matrix can be produced that approximately recovers the quadratic regularizer defined by the $\Ocal(m^2)$ size importance matrix. 
Here, $t \ll m$ is a tuneable parameter that balances the computation cost and matrix size with the quality of approximation, and can be chosen as a small number in practice. Our method, called \emph{sketched SR}, has the following notable advantages:
\begin{enumerate}
    \item Has a \emph{theoretically guaranteed} small approximation error, provided the importance matrix has a well-behaved spectrum, e.g. has low effective rank. Fortunately, for deep neural network and commonly used SR methods, the importance matrix (e.g. empirical Fisher) does indeed have low (effective) rank \citep{sagun2017empirical,chaudhari2018stochastic}, but is not diagonal (see Figure \ref{fig:fim_illustration}). 
    
    \item Is \emph{algorithm oblivious} by construction, i.e. for any algorithm that belongs to the structural regularization framework (defined in Section \ref{sec:SR}), a sketched version can be readily established without additional, algorithm specific considerations.
    
    \item Is \emph{computationally efficient} and \emph{easy to implement}. Both sketched SR and diagonal SR make only one pass of the data (of the old task) to obtain the approximation. Though sketched SR saves $\Ocal(tm)$ parameters, which is slightly larger than the $\Ocal(m)$ parameters in diagonal SR.
    This additional cost is easily affordable as setting $t\le 50$ is sufficient for sketched SR to outperform diagonal SR in our experiments.
    
    \item \emph{Consistently outperforms} its diagonal counterpart on overcoming catastrophic forgetting, in both synthetic experiments and benchmark lifelong-learning tasks, including permuted-MNIST and CIFAR-100. 
\end{enumerate}

\paragraph{Paper Layout.}
The remaining part of this paper is organized as follows:
the related literature is reviewed in Section \ref{sec:relates}; in Section \ref{sec:SR} we formally introduce the structural regularization from the geometric viewpoint \citep{kolouri2019sliced,ritter2018online}; then in Section \ref{sec:SSR}, we present our sketched structural regularization, its practical implementation and its theoretical properties; in Section \ref{sec:exps}, we experimentally compare our methods with the diagonal counterparts, which verifies the effectiveness of our methods; finally, this paper is concluded by Section \ref{sec:conclusion}.

\section{Related Works}\label{sec:relates}
\paragraph{Functional Regularization.} 
Besides structural regularization, another popular category of approaches to overcome catastrophic forgetting is \emph{functional regularization} \citep{jung2016less,li2017learning,rannen2017encoder,shin2017continual,hu2018overcoming,rozantsev2018beyond,wu2018memory,li2019learn}.
Similar to structural regularization, functional regularization also adds regularizer (when training new tasks) to penalize the forgetting of useful old knowledge; however, functional regularization may use very general (hence, functional) regularizers, in addition to quadratic ones. 
For example, \citet{jung2016less,li2017learning} snapshot a teacher model that learned from old tasks, and use it to regularize a student model that fits new tasks.
Moreover, generative models are applied to generate pseudo-data (\emph{memory}) from old tasks, and the pseudo-data is mixed to the new data distribution as a regularization (\emph{replay}) for learning new tasks \citep{rannen2017encoder,rostami2019complementary,shin2017continual,wu2018memory,hu2018overcoming}. This is also known as \emph{memory replay}.
Finally, we remark that functional regularization can be used together with structural regularization \citep{shin2017continual,rozantsev2018beyond}.
Our focus of this paper is to use linear sketching methods to improve SR methods; an interesting future work is to apply similar ideas (e.g., coresets \citep{10.1145/1993636.1993712,10.1145/1007352.1007400}) to improve functional regularization methods, especially for those based on memory replay.

\paragraph{Non-Diagonal Importance Matrix.}
Diagonal approximation is a crude, but de facto approach to compress the full IM in most of existing SR algorithms~\citep{kolouri2019sliced,aljundi2018memory,kirkpatrick2017overcoming,chaudhry2018riemannian,zenke2017continual}.
Before this paper, there are a few works that study structural regularization with non-diagonal IM \citep{liu2018rotate,ritter2018online}, which we discuss in sequence.
\citet{ritter2018online} adopt the layer-wise block-diagonal approximation as a better replacement to the commonly used diagonal version for the importance matrix: even so, the cross-layer weight dependence is being ignored; moreover, in our experiments, block-diagonal empirical Fisher is not a good approximation to empirical Fisher matrix, either (see Figure \ref{fig:fim_illustration}).
\citet{liu2018rotate} propose layer-wise rotation of the empirical Fisher such that the new matrix can be more diagonal-alike; this procedure not only assumes cross-layer independence (of weights), but even assumes independence between layer inputs and layer gradients (see Eq. (7) in \citep{liu2018rotate}).
In comparison, the sketching methods adopted in this paper only require a very weak assumption, i.e., the importance matrix has low effective rank.

\paragraph{Linear Sketching.}
Linear sketching is a widely studied technique for dimensionality reduction. We rely on the popular sketching method \emph{CountSketch} \citep{charikar2002finding} that has its roots in the Johnson-Lindenstrauss transform. Randomized linear sketching methods, such as CountSketch, draw a random matrix $S \in \Rbb^{t \times m}$ and embed the columns of the input matrix $W \in \Rbb^{n \times m}$ into a smaller dimension $t \ll n$ by outputting $SW$. By carefully constructing the random distribution, it can be shown that the sketch $SW$ \emph{preserves the norms of the vectors in the subspace spanned by the columns of $W$} up to some error. Such sketching techniques are known as oblivious subspace embeddings (OSEs). This property of OSEs makes them a natural tool for approximating the quadratic regularizer in SR methods.

Sparse OSE methods \citep{nelson2013osnap,cohen2016simpler} such as CountSketch have a two-fold advantage: i) they're \emph{oblivious}, which means that the random distribution is defined independent of the input matrix $W$ and ii) the sketch $SW$ can be computed in time that is linear in the input size (e.g. proportional to the number of non-zero entries in $W$). These methods have have been widely used, giving fast algorithms for various problems such as low-rank approximation, linear regression \citep{sarlos2006improved,clarkson2017low, meng2013low}, k-means clustering \citep{cohen2015dimensionality}, leverage score estimation \citep{drineas2012fast} and numerous other problems \citep{lee2019solving,ahle2020oblivious,vandenbrand_et_al:LIPIcs:2021:13602}.

\section{Preliminaries}\label{sec:SR}
We use $(x,y) \in \Rbb^{s} \times \Rbb^k$ to denote a feature-label pair, and $\theta \in \Rbb^m$ to denote the model parameter.
A parametric model is denoted by $\phi(\cdot\, ; \theta) : \Rbb^s \rightarrow \Rbb^k$.
Given a distance measure of two distributions, $d(\cdot, \cdot)$, the individual loss over data point $(x,y)$ can be formulated as
\[\ell(x,y; \theta) := d( \phi(x;\theta), y).\]
For example, in deep neural networks, $\phi(\cdot\, ; \theta)$ is the network output, and $d(\cdot, \cdot)$ is usually chosen to be the cross entropy loss \citep{goodfellow2016deep}.

\paragraph{Structural Regularization.}
Let task $A$ with data distribution $(x,y)\sim \Dcal_A$ be an already well-learned task on network $\phi$ with learnt parameters $\theta^*_A$.
In order to overcome catastrophic forgetting when learning a new task $B$, with data distribution $(x,y)\sim \Dcal_B$, structural regularization algorithms apply an extra regularizer $\Rcal(\theta)$ to the main loss and optimize the following total loss: 
\begin{align*}
    \arg\min_\theta \Ebb_{(x,y) \sim \Dcal_B} \sbracket{\ell (x,y; \theta)} + \lambda \cdot \Rcal(\theta).
\end{align*}
Here, the expectation should be understood as the empirical expectation over the training set.
As for the regularization term, $\lambda$ is a hyper-parameter, and $\Rcal (\theta)$ is a quadratic regularizer that penalizes the weight for being deviated from $\theta^*_A$, the learnt weight from the previous task $A$:
\[
\Rcal(\theta) := \half (\theta - \theta^*_A)^\top \Omega (\theta - \theta^*_A),
\]
where $\Omega \in \Rbb^{m\times m}$ is an importance matrix and is positive semi-definite (PSD).
As we will see shortly, the PSD matrix $\Omega$ usually has a natural decomposition as \citep{kirkpatrick2017overcoming,aljundi2018memory}: 
\begin{align}\label{eqn:omega_decomposed}
    \Omega = \frac{1}{n}{W^\top W},
\end{align}
where each row of $W \in \Rbb^{n \times m}$ is a Jacobian matrix of a certain individual loss (which might not be the one used for the main loss) of data $x$ from task $A$, and $n$ is the number of training data in task $A$.
Then, the structural regularizer $\Rcal(\theta)$ can be written as   
\begin{align}\label{eqn:regularizer_normVersion}
    \Rcal (\theta) = \frac{1}{2n} \|{W} \cdot (\theta - \theta^*_A)\|_2^2, \quad W \in \Rbb^{n\times m}.
\end{align}

\begin{table*}[t]
\caption{The construction of the importance matrices in EWC \citep{kirkpatrick2017overcoming} and MAS \citep{aljundi2018memory}. 
In the EWC row, the loss function $\ell(x,y; \theta) := d( \phi(x;\theta), y)$ is defined with $d(\cdot\,,\cdot)$ being the cross entropy loss.
The rows of matrix $W \in \Rbb^{n \times m}$ are indexed by the input data from task $A$.
}\label{table:selectivePlasticitySummary}
    \vskip 0.15in
    \centering\
    \begin{sc}
    \begin{tabular}{c|cc}
        \toprule
         & Matrix $\Omega$ & Row-Vector $(W)_x$ \\
        \midrule
        EWC & $\Ebb_{(x,y)\sim \Dcal_A} \nabla_{\theta} \ell (x,y;\theta^*_A) \cdot \nabla_{\theta} \ell (x,y;\theta^*_A)^\top $ & $\nabla_{\theta} \ell (x,y;\theta^*_A)$
        \\
        MAS & $\Ebb_{x\sim \Dcal_A} \big(\nabla_{\theta}\norm{\phi(x;\theta^*_A)}^2_2\big)\cdot \big(\nabla_{\theta}\norm{\phi(x;\theta^*_A)}^2_2 \big)^\top $ & $\nabla_{\theta}\norm{\phi(x;\theta^*_A)}^2_2$\\
        \bottomrule
    \end{tabular}
    \end{sc}
    \vskip -0.1in
\end{table*}

\paragraph{Two Examples.}\label{sec:sketching_eqc_mas_scp}
Table \ref{table:selectivePlasticitySummary} summarizes two examples for the importance matrices in: \emph{Elastic Weight Consolidation} (EWC) \citep{kirkpatrick2017overcoming} and \emph{Memory Aware Synapses} (MAS) \citep{aljundi2018memory}.
It is worth noting that the importance matrix used in EWC is the \emph{empirical Fisher} evaluated at the optimal weight for task $A$.

\paragraph{Diagonal Approximation.}
Unfortunately, both matrices $\Omega$ and $W$ have $m^2$ and $mn$ entries respectively, which makes them prohibitively large to compute and store for big models like deep neural networks. 
As a compromise, practitioners often take the diagonal of $\Omega$ as an approximation. This leads to the presented version of EWC \citep{kirkpatrick2017overcoming} and MAS \citep{aljundi2018memory} in their original paper.
These are called \emph{diagonal EWC} and \emph{diagonal MAS} respectively in this paper to be distinguishing with our variants.
However, as we have discussed and demonstrated on a synthetic dataset, such a treatment ignores the dependence between weights and exacerbates performance degeneration for overcoming catastrophic forgetting.
In the following we present our sketched version of the above algorithms, which can make use of the off-diagonal entries of $\Omega$ to improve the diagonal approximated version.

\section{Sketched Structural Regularization}\label{sec:SSR}
In this section we propose our framework of sketching the regularizer from \eqref{eqn:regularizer_normVersion} and describe the specific sketch construction along with some theoretical guarantees. We describe our construction in terms of the general framework of structural regularization for lifelong learning from Section \ref{sec:SR}.
Then we contrast our approximation method with other compression methods like PCA. 
Finally we describe how we go from the two-task settings to an online version of the algorithm in a way that is standard in works on structural regularization \citep{kolouri2019sliced,chaudhry2018riemannian,schwarz2018progress}.

\paragraph{Sketched Regularizer.}
We propose a method to sketch the matrix $\Omega$ from \eqref{eqn:omega_decomposed} by reducing the dimensionality of each of the matrix $W$ from $n$ dimensions to $t$ dimensions for a $t \ll \min\{n,m\}$. Specifically, we draw a random matrix $S \in \Rbb^{t \times n}$ from a carefully chosen distribution and approximate the regularizer \eqref{eqn:regularizer_normVersion} in SR methods with 
\begin{align}\label{eqn:sketched_regularizer}
\widetilde{\Rcal}(\theta) = \frac{1}{2n}\| \widetilde{W}\cdot (\theta - \theta^*_A)\|_2^2, \ \ \widetilde{W} := S W\in \Rbb^{t\times m}.
\end{align}

We use \emph{CountSketch} \citep{charikar2002finding} to construct the sketched matrix $\widetilde{W} = SW$, which is formally presented in Algorithm \ref{alg:SSR}.
CountSketch reduces the number of rows (aka, the dimension of the columns) of $W$ by the following: first the rows of $W$ are randomly partitioned into $t$ groups (Algorithm \ref{alg:SSR}, line \ref{line:partitioning}), then rows in each group are randomly, linearly combined (with random signs as weights) into a single new row (Algorithm \ref{alg:SSR}, line \ref{line:updateW}).

Two remarks are in order for the practical implementation of Algorithm \ref{alg:SSR}:
(i) note that Algorithm \ref{alg:SSR} only makes one pass of the data from task $A$, which is as required for computing diagonal approximation;
(ii) note that Algorithm \ref{alg:SSR} requires $\Ocal(t)$ times auto-differentiation, but since $t$ is small and the sketch construction only needs to done once per new task, the cost is affordable in practice (see more in Section \ref{sec:exps}).

\paragraph{Comparison with Low-Rank Approximation Methods.}
The main advantage of using CountSketch over more complicated low-rank approximation methods (e.g. PCA) to compress the importance matrix in SR methods, is that it can be computed with only a small amount of additional computation and only a modest blow-up in memory compared to the diagonal approximation.
However PCA is usually computational intractable for big models such as deep neural networks.
Moreover, in below, we show CountSketch achieves provable small approximation error (for matrix with low stable-rank), as can be guaranteed by PCA. 

\begin{algorithm}[tb]
  \caption{Sketch Construction in Sketched SR}
  \label{alg:SSR}
\begin{algorithmic}[1]
  \STATE {\bfseries Input:} Data from task $A$ and optimized neural network $\phi( \cdot\, ; \theta_A^*)$ for task $A$
  \STATE {\bfseries Parameters:} Size of sketch $t \in \Nbb^+$
  \STATE Initialize 2-wise and 4-wise independent hash functions $h : [n] \rightarrow [t]$ and $\sigma : [n] \rightarrow \{-1, 1\}$ respectively \label{line:hash}
  \FOR{$k=1,\dots,t$}
  \STATE Group data $G_k :=\{x\in A : h(x) = k \}$ \label{line:partitioning}
  \STATE Compute $\sum_{x\in G_k}\sigma(x)(W)_x$ as per Table \ref{table:selectivePlasticitySummary} by auto-differentiation
  \STATE Set $(\widetilde{W})_{k} \leftarrow \sum_{x\in G_k}\sigma(x)(W)_x$\label{line:updateW}
  \ENDFOR
  \RETURN $\widetilde{W}\in \Rbb^{t \times m}$
\end{algorithmic}
\end{algorithm}

\paragraph{Theoretical Properties.}
The following theorem from \citet{cohen2016OSEStableRank} builds on several results on CountSketch matrices, giving theoretical guarantees for sketching quadratic forms of matrices.

The theorem is re-phrased for our purposes, showing the quality of approximation by the sketch in preserving $\ell_2$-norms of vectors in the subspace spanned by the columns of $W$, the matrix that is being sketched. There is a trade-off in the quality of approximation by the sketch and its size, given by the dimension of the columns $t$. In particular, the error in preserving the $\ell_2$-norm of any $W\theta$ depends on the spectrum of $W$; when $t \geq \|W\|_F^4/(\epsilon^2\|W\|_2^4)$ the error is \emph{additive} and scales with $\epsilon\|W\|_2^2\|\theta\|_2^2$, which we detail in the following theorem. We state a full-version of the theorem showing the exact trade-off between the the number of buckets $t$ and the quality of approximation in Appendix \ref{sec:append-theory}.

\begin{figure}
    \centering
    \includegraphics[width=0.35\linewidth]{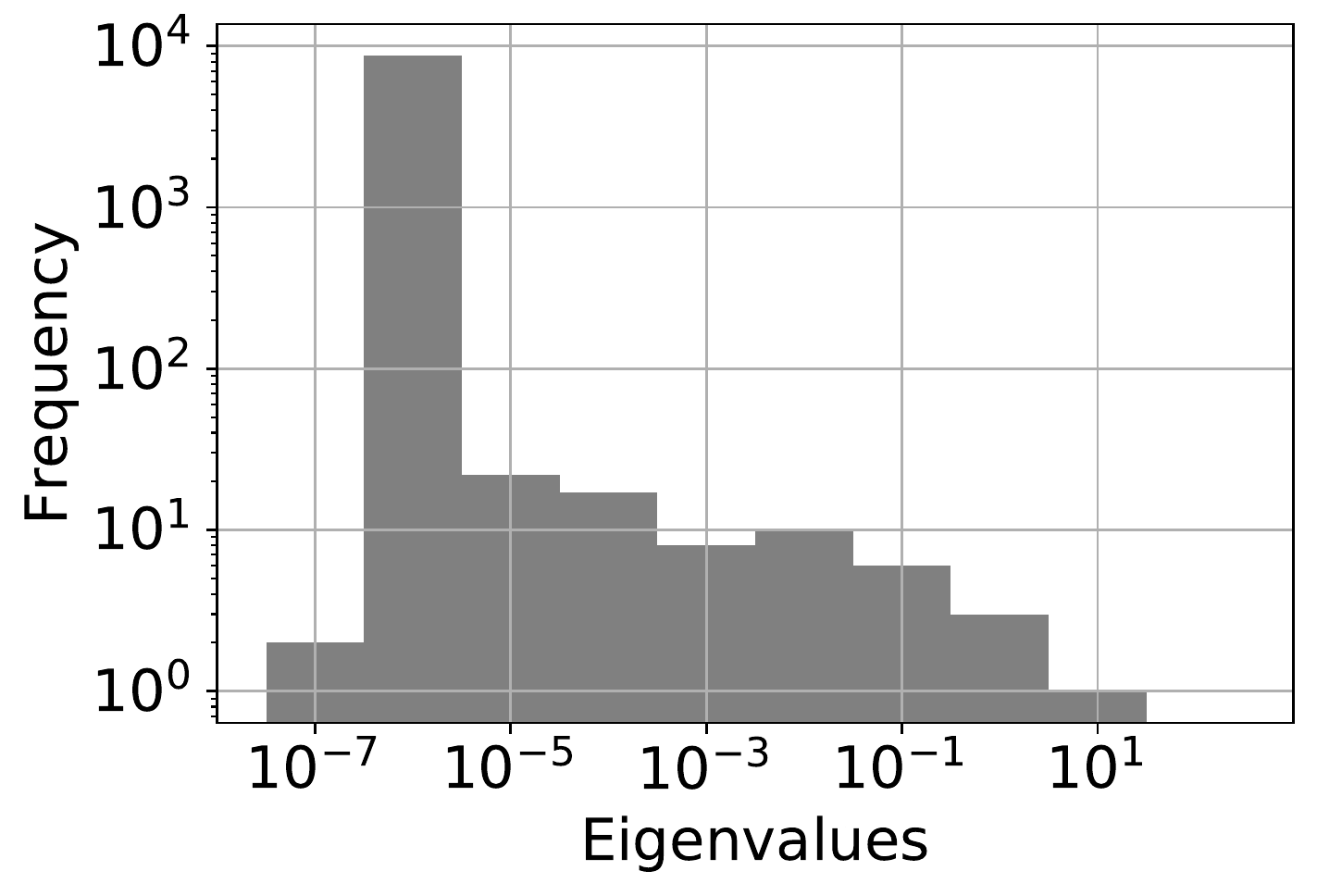}
    \caption{The spectrum of the empirical Fisher studied in Figure \ref{fig:fim_illustration}.
    The empirical Fisher is $8,770 \times 8,770$ and its stable rank is $1.26\pm 0.13$ (over $5$ random seeds). Both the plotted spectrum and the computed stable rank show that this empirical Fisher has a very small effective rank.
    } 
    \label{fig:spectrum}
\end{figure}

\begin{thm}\label{cor:sketch_guarantee}
For a matrix $W \in \Rbb^{n \times m}$ with stable rank\footnote{The stable rank of a matrix $W$ is $\|W\|_F^2/\|W\|_2^2$.} $r$, a CountSketch matrix $S \in \Rbb^{t \times n}$ with $t = \Ocal(r^2/\epsilon^2)$ has the property that with probability at least $0.99$, 
\[
    {\abs{ \|SW \theta\|_2^2 - \|W \theta\|_2^2 }} \leq {\epsilon \cdot \|W\|_2^2 \cdot \|\theta\|_2^2}
\]
for all vectors $\theta \in \Rbb^m$.
\end{thm}

Notice that stable rank never exceeds the usual rank, and can be significantly smaller when the matrix has a decaying spectrum. The importance matrix considered in SR methods usually have fast decaying spectrum (see Figure \ref{fig:spectrum}), i.e. small stable rank, making it effective to use CountSketch to approximate quadratic forms with the matrices. 
For instance, in the synthetic experiment we considered, the stable rank of the empirical Fisher shown in Figure \ref{fig:toy_EF} is $1.26$ with standard deviation $0.13$, measured over 5 trials. Note that the empirical Fisher is $8,770 \times 8,770$.

\paragraph{Online Extension of Sketched SR.}
Lifelong learning often requires learning more than two tasks sequentially. One method of extending the Sketched SR method to learn on multiple tasks to maintain separate sketches for each task and compute the regularizer $\widetilde{\Rcal}(\theta)$ in \eqref{eqn:sketched_regularizer} from each of the previous tasks when learning the current one. This approach would cause the memory requirement to grow linearly in the number of tasks and can become a bottleneck in scaling the method. A standard way to tackle this is in works on structural regularization is to apply the \emph{moving average} method to aggregate the histories \citep{chaudhry2018riemannian,schwarz2018progress}. Specifically, let $\widetilde{\Omega}_{\tau-1}$ be the importance matrix maintained after training on the $(\tau - 1)$-th task, then, given the (approximate) importance matrix $\widetilde{\Omega}$ outputted on the data from task ${\tau}$, the histories are updated as
\begin{align}\label{eqn:omega_online_learning}
\widetilde{\Omega}_\tau \leftarrow \alpha \widetilde{\Omega} + (1-\alpha)\widetilde{\Omega}_{\tau - 1}
\end{align}
where $\alpha \in (0, 1]$ is a hyperparameter.

Since the matrix $\widetilde{\Omega}$ is a diagonal matrix for each task in the aforementioned methods, computing the sum from \eqref{eqn:omega_online_learning} is straightforward. Sketched SR, however, doesn't explicitly compute the matrix $\widetilde{\Omega} = \widetilde{W}^\top \widetilde{W}$, hence we cannot hope to compute the matrix $\widetilde{\Omega}_\tau$ defined by the sum in  \eqref{eqn:omega_online_learning}. We propose the following method:
let $\widetilde{W}_{\tau-1}$ be the maintained sketch after training on the $(\tau - 1)$-th task, then, given the weight $\theta^*$ and the sketch $\widetilde{W}$ outputted on the data from task ${\tau}$, we update the importance matrix as
\begin{align}\label{eqn:sketched_online_learning}
\widetilde{W}_\tau \leftarrow \sqrt{\alpha} \widetilde{W} + \sqrt{1-\alpha}\widetilde{W}_{\tau - 1}.
\end{align}
When learning on task $\tau+1$ we use the regularizer 
\begin{align}\label{eqn:sketched_online_learning_regularizer}
\widetilde{\Rcal}_\tau(\theta) := \frac{1}{2n} \|\widetilde{W}_\tau (\theta-\theta^*) \|_2^2.
\end{align}
A priori, it is not clear why the regularizer from \eqref{eqn:sketched_online_learning_regularizer} is a good approximation to that induced by the importance matrix from \eqref{eqn:omega_online_learning}. We give a theorem, along with a proof in Appendix \ref{sec:append-online}, that implies that for any fixed $\theta \in \Rbb^m$ the regularizer given by \eqref{eqn:sketched_online_learning_regularizer} is close to that induced by the importance matrix $\widetilde{\Omega}_\tau$ from \eqref{eqn:omega_online_learning}.

\begin{thm}\label{cor:online_guarantee}
 Let $W_1, \dots, W_\tau \in \Rbb^{n \times m}$ be a sequence of matrices, $\alpha_1, \dots, \alpha_\tau \geq 0$ be a sequence of weights , and $S_1, \dots, S_\tau \in \Rbb^{t \times n}$ be a sequence of independent CountSketch matrices with sketch size $t \in \Nbb^+$. There exists a constant $C>0$ such that for any fixed $\theta \in \Rbb^m$, 
\[
    {\abs{ \left\| \left( \sum_{i=1}^\tau \sqrt{\alpha_i} S_i W_i \right) \theta \right\|_2^2 - \sum_{i=1}^\tau \alpha_i \|W_i \theta \|_2^2 }} \leq {\frac{C}{\sqrt{t}} \cdot \sum_{i=1}^\tau \alpha_i \|W_i \theta \|_2^2}
\]
with probability at least $0.99$.
\end{thm}
As a corollary to Theorem \ref{cor:online_guarantee}, we show the approximation error of the regularizer from \eqref{eqn:sketched_online_learning_regularizer}. 
\begin{cor}
    For any fixed $\theta \in \Rbb^m$, the regularizer $\widetilde{\Rcal}\tau(\theta)$ given by \eqref{eqn:sketched_online_learning_regularizer} has the property that
    \[
        \widetilde{\Rcal}\tau(\theta) = \left(1 \pm \bigO{\frac{1}{\sqrt{t}}} \right) \cdot (\theta - \theta^*)^\top \widetilde{\Omega}_\tau(\theta - \theta^*)
    \]
    with probability 0.99 and where $\widetilde{\Omega}_\tau$ is the matrix given by the recurrence in \eqref{eqn:omega_online_learning} with  $\widetilde{\Omega} = \widetilde{W}^\top \widetilde{W}$.
\end{cor}

\paragraph{Remark.} Note that Theorem \ref{cor:sketch_guarantee} enjoys a stronger guarantee than that of Theorem \ref{cor:online_guarantee}, i.e., while the approximation guarantee in Theorem \ref{cor:online_guarantee} holds for \emph{any fixed} vector $\theta \in \Rbb^m$, the guarantee in Theorem \ref{cor:sketch_guarantee} holds for all $\theta \in \Rbb^m$ \emph{simultaneously}. We expect that the stronger guarantee of Theorem \ref{cor:sketch_guarantee} can be achieved in the setting of Theorem \ref{cor:online_guarantee} by computing $\bigO{\log(t)}$ independent copies of the aggregated sketch $\widetilde{W}_\tau$ from \eqref{eqn:sketched_online_learning}. The regularizer used when learning on task $\tau + 1$ is simply the average of the regularizer $\widetilde{\Rcal}_\tau(\theta)$ from \eqref{eqn:sketched_online_learning_regularizer} outputted by each copy of $\widetilde{W}_\tau$. We leave it to future work to analyze this extension of Sketched SR in order to obtain the stronger guarantee for the setting in Theorem \ref{cor:online_guarantee}.

\section{Experiments}\label{sec:exps}
In this section, we present empirical evidence that verifies the effectiveness of our proposed Sketched SR methods.
The experiments are conducted with variants of two representative SR algorithms, EWC \citep{kirkpatrick2017overcoming} and MAS \citep{aljundi2018memory}.
All the reported numerical results are averaged over $5$ runs with different random seeds.

\subsection{Synthetic Experiments}\label{sec:toy-exp}
We start with a series of synthetic experiments.

\paragraph{Setup.} 
We first consider a synthetic 2D binary classification task from \citet{pan2020continual}. The experiment consists of 5 classification tasks learnt sequentially using the regularization induced by each of EWC and MAS with a small multi-layer perceptron.
The network has $8,770$ parameters.
For the regularization matrix induced by EWC and MAS, we compare the performance of various approaches to approximating the matrix including: 
\begin{enumerate}[label=(\roman*),noitemsep,topsep=0mm,parsep=0mm,partopsep=0mm,leftmargin=*]
    \item a diagonal approximation;
    \item a block-diagonal approximation, with a sequence of $ 50 \times 50 $ non-zero blocks along the diagonal;
    \item sketched SR with sketch size $t = 50$;
    \item a rank-$50$ SVD;
    \item and the full importance matrix.
\end{enumerate}
For all algorithms, we use ADAM as the optimizer with learning rate $ 10^{-3} $, and use the moving average parameter $ \alpha = 0.5 $. 
For more details please see Appendix \ref{sec:append-toy}.

\paragraph{Approximation vs Full Matrix Comparison.}
We first plot the empirical Fisher (the importance matrix in EWC methods) and the sketched empirical Fisher in Figure \ref{fig:fim_illustration}. The empirical Fisher is obtained with the optimal weight that fits the first four tasks and the sketched empirical Fisher uses sketch size $t=50$.
From the figure we observe that the empirical Fisher cannot be well-approximated by its diagonal or block-diagonal; moreover, the sketched empirical Fisher can utilize the off-diagonal entries to generate a better approximation.
This is further supported by the numerical approximation error shown in the table within Figure \ref{fig:fim_illustration}.
Note that while the low-rank method can offer a better approximation, it is not computationally efficient in practice.

\paragraph{Performance of the Compared Algorithms.}
We then compare the performance of each algorithms in Figures \ref{fig:toy_ewc_illustration}.
The plots consistently indicate that sketched SR methods are more effective than diagonal SR methods for overcoming catastrophic forgetting. Additionally, while low-rank SR and full SR perform better than sketched SR, they are not computationally feasible in practical settings with large models.

\begin{figure*}
    \centering
    \subfigure[Diagonal  ]{\includegraphics[width=0.19\linewidth]{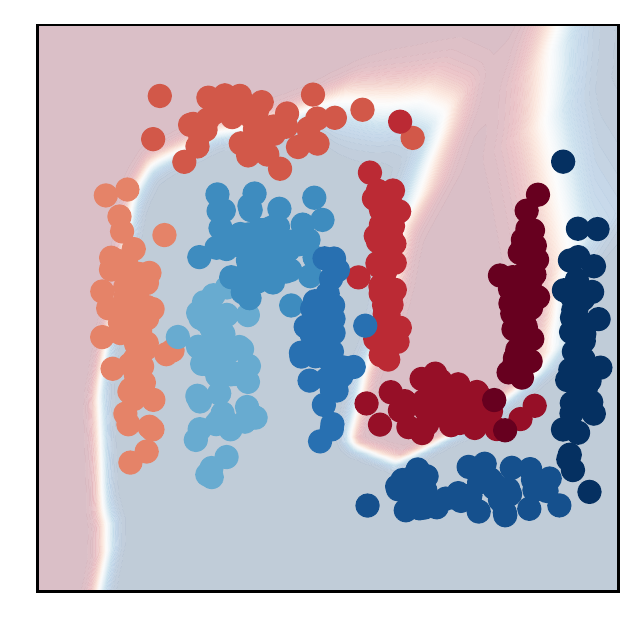}\label{fig:toy_ewc_final_boundary}}
    \subfigure[Block-Diagonal ]{\includegraphics[width=0.19\linewidth]{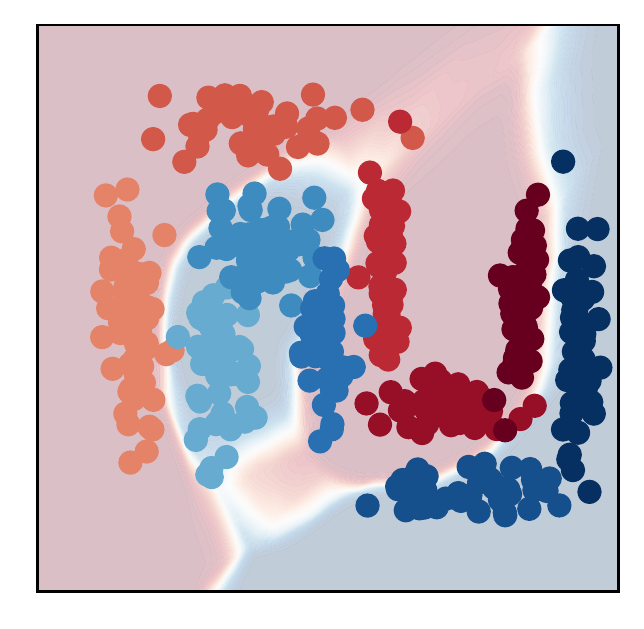}\label{fig:toy_block_diagonal_ewc_final_boundary}}
    \subfigure[Sketched]{\includegraphics[width=0.19\linewidth]{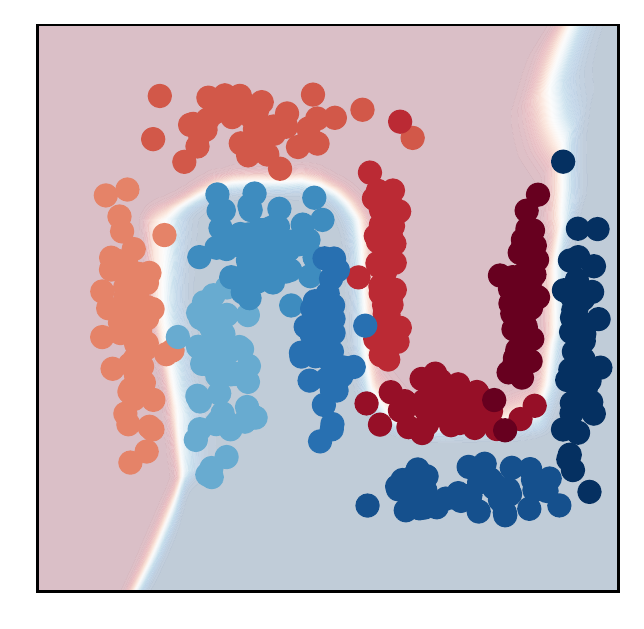}\label{fig:toy_sketch_ewc_final_boundary}}
    \subfigure[Low-Rank ]{\includegraphics[width=0.19\linewidth]{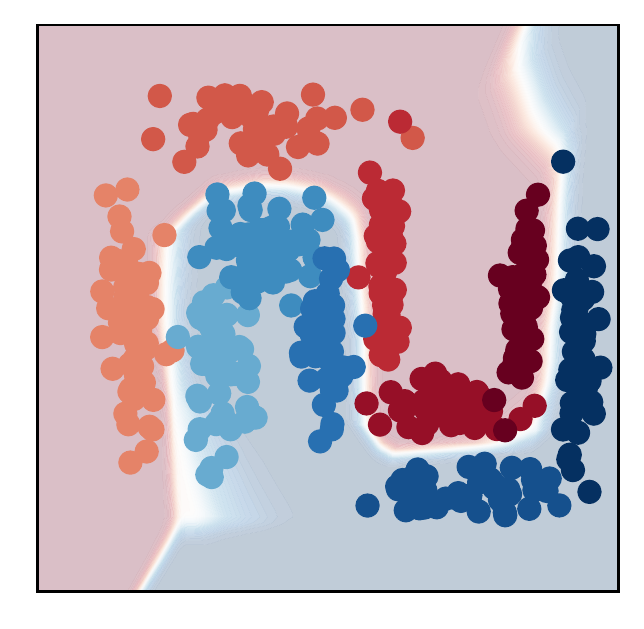}\label{fig:toy_low_rank_ewc_final_boundary}}
    \subfigure[Full EWC]{\includegraphics[width=0.19\linewidth]{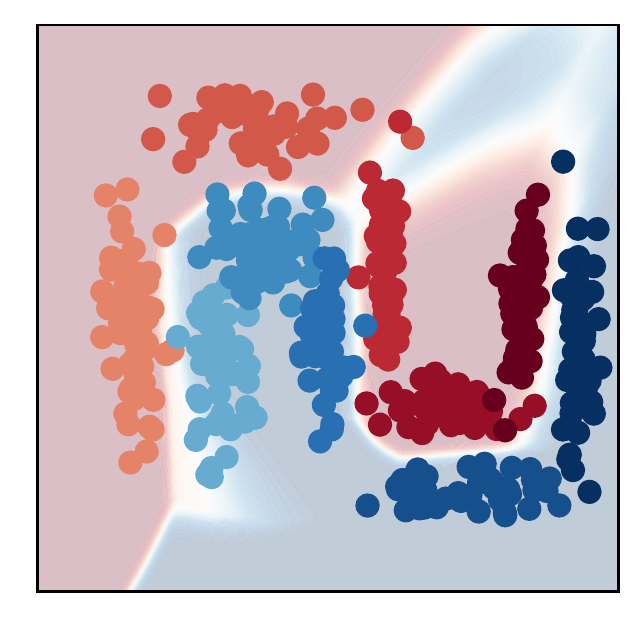}\label{fig:toy_full_ewc_final_boundary}}
    \label{fig:toy_ewc_illustration}
    
    \begin{subfigure}
        \centering
        \begin{small}
        \begin{sc}
        \begin{tabular}{c|ccccc}
            \toprule
            Avg. Accuracy & Diagonal & Block-Diag. & Sketched & Low-Rank & Full IM \\
            \midrule
            EWC & $ 88.0 \pm 7.4\% $ & $ 79.3 \pm 9.8\% $ & $ 92.1 \pm 6.8\% $ & $ 93.9 \pm 3.9\% $ & $ 95.7 \pm 5.4\% $ \\
            MAS & $ 83.1 \pm 5.6\% $ & $ 76.6 \pm 12.0\% $ & $ 85.9 \pm 9.6\% $ & $ 84.2 \pm 9.8\% $ & $ 89.9 \pm 7.2\% $ \\
            \bottomrule
        \end{tabular}
        \end{sc}
        \end{small}
    \end{subfigure}\label{table:toy_accAndFrobErr}
    
    \caption{ Variants of EWC \citep{kirkpatrick2017overcoming} and MAS \citep{aljundi2018memory} on a synthetic 2D binary classification task from \citep{pan2020continual}. 
    In the figures, the two classes are represented by the different shades of red/blue, learnt sequentially using the variants of EWC.
    For the block-diagonal SR, the block size is $ 50 \times 50 $; for the sketched SR, we set $t=50$; for the low-rank SR, the rank is $50$.
    The figures show the the decision boundaries found by the compared algorithms.
    The table shows the average accuracy across all tasks (after learning the final task) for the compared algorithms.
    The plots and the table suggest that: sketched SR has a higher average accuracy than both diagonal SR and block-diagonal SR by overcoming catastrophic forgetting; while average accuracy of low-rank SR and full SR is higher, they requires significantly more computation which is not affordable in practice.
    See Section \ref{sec:toy-exp} for more details.
    }
\end{figure*}

\subsection{Permuted-MNIST}\label{sec:mnist-exp}
Next we demonstrate the effectiveness of our methods with experiments on permuted-MNIST.

\paragraph{Setup.}
In this benchmark experiment for lifelong learning \citep{kirkpatrick2017overcoming,zenke2017continual,rostami2019complementary,ritter2018online,ramasesh2021anatomy}, there are $10$ sequential tasks, each of them is a $10$-classes classification task based on a permuted MNIST dataset, where the pixels in each figure are permuted according to certain rule (to be more specific, the permutation rule is same within a task but random across different tasks).
We use a multi-layer perceptron as the classifier, and ADAM as the optimizer.
The learning rate is set to be $ 10^{-4} $.
The moving average parameter is $ \alpha = 0.25 $ for all algorithms.
For each compared algorithm, the regularization coefficient $ \lambda $ is chosen to be optimal by grid search.
All the reported numerical results are averaged over $5$ runs with different random seeds.
For more details see Appendix \ref{sec:append-mnist}.

\paragraph{Performance of the Compared Algorithms.}
Figure \ref{fig:perm_mnist_compare} shows the average accuracy across previously learned tasks after each epoch of training for the compared methods. Table \ref{table:mnist_cifar_result} reports the averaged accuracy (across all tasks) of the compared algorithms.
From the figures and the table, we consistently see that sketched SR methods outperform their diagonal counterparts, in both EWC and MAS regimes, in terms of overcoming catastrophic forgetting.
This is explored deeper in Figure \ref{fig:perm_mnist_task_compare_1}, where we show the accuracy on each task after training on all the tasks for the compared algorithms.
According to Figure \ref{fig:perm_mnist_task_compare_1}, sketched SR methods forget less about the early tasks, which directly demonstrate its advantage for overcoming catastrophic forgetting.
This is consistent to our finding from the synthetic experiments.

\begin{figure}[t]
    \centering
    \includegraphics[width=0.6\columnwidth]{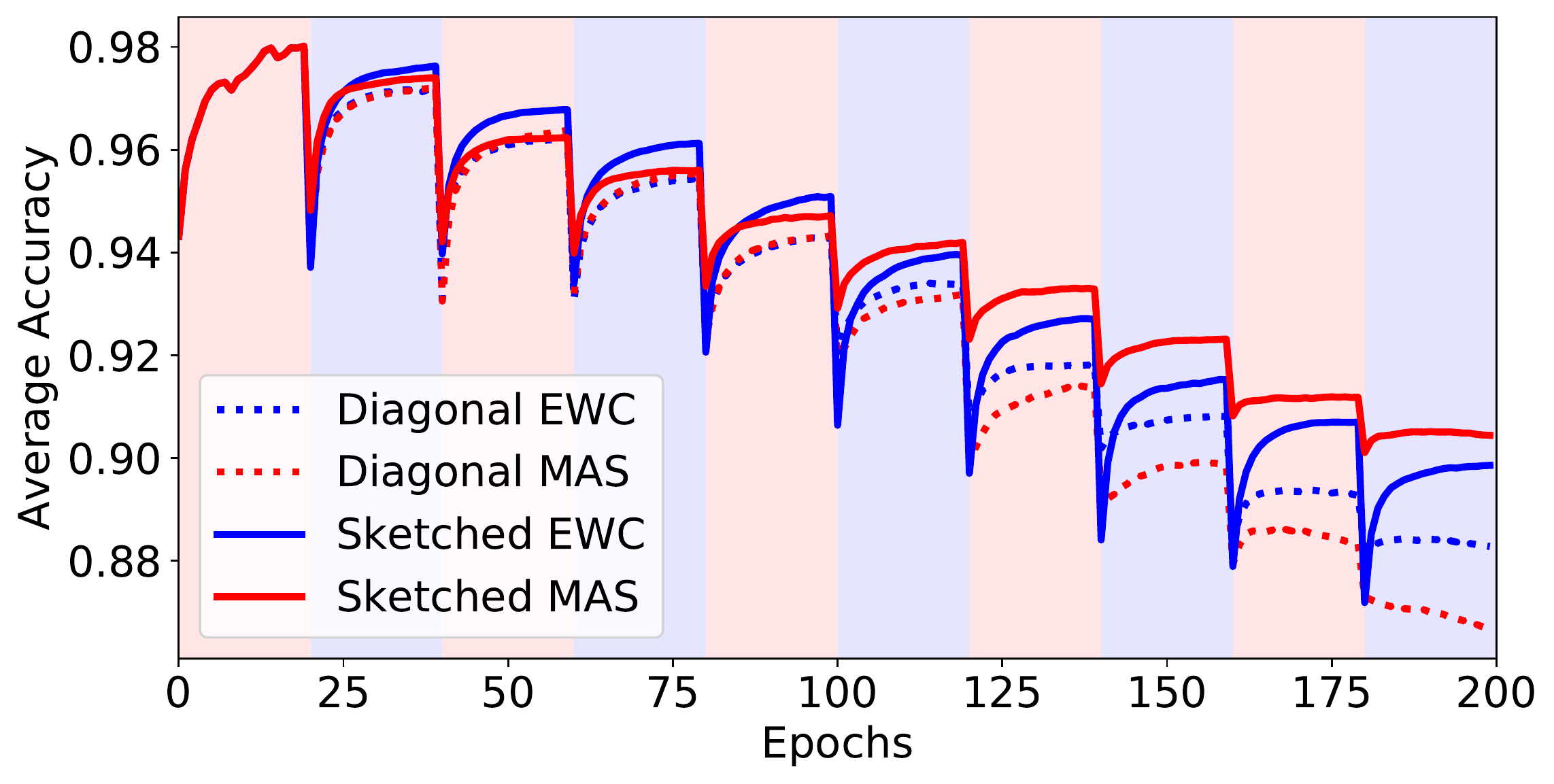}
    \caption{
    The average accuracy across previously learned tasks after each epoch of training for both diagonal and sketched methods on permuted-MNIST.
    The plots suggest that sketched methods consistently outperform their diagonal counterparts for overcoming catastrophic forgetting.
    See Section \ref{sec:mnist-exp} for more details.}
    \label{fig:perm_mnist_compare}
\end{figure}

\begin{figure}[t]
    \centering
    \begin{minipage}[b]{.48\textwidth}
        \centering
        \begin{subfigure}
            \centering
            \includegraphics[width=\columnwidth]{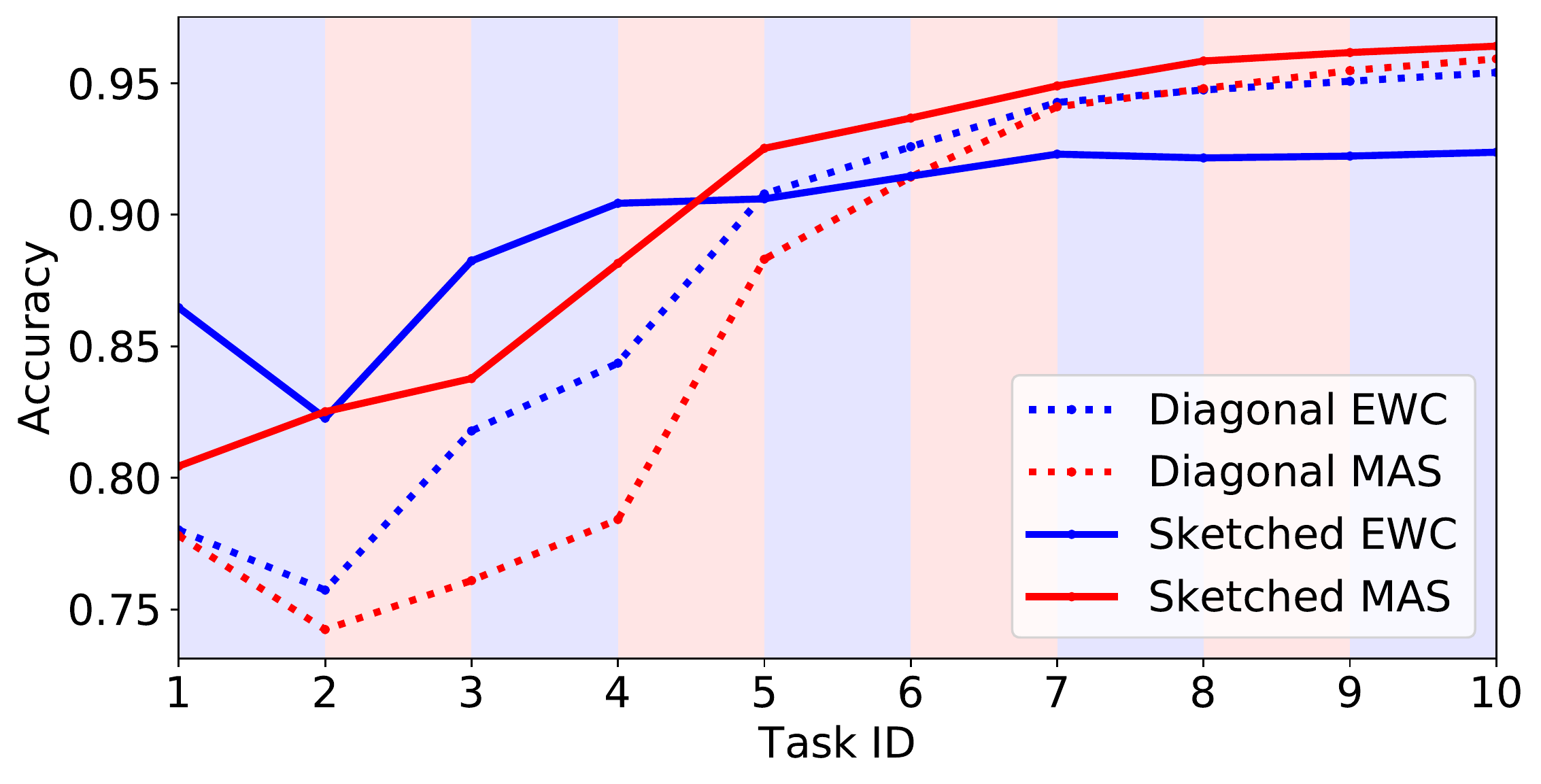}
            \caption{The accuracy of each task (after training on all tasks) of sketched methods vs. diagonal methods on permuted-MNIST. 
            The plots show that, sketched SR methods significantly outperform their diagonal counterparts in early tasks, which suggests that sketched SR methods are superior for overcoming catastrophic forgetting.}
            \label{fig:perm_mnist_task_compare_1}
        \end{subfigure}
    \end{minipage}\hfill
    \begin{minipage}[b]{.48\textwidth}
        \centering
        \begin{subfigure}
            \centering
            \includegraphics[width=0.6\columnwidth]{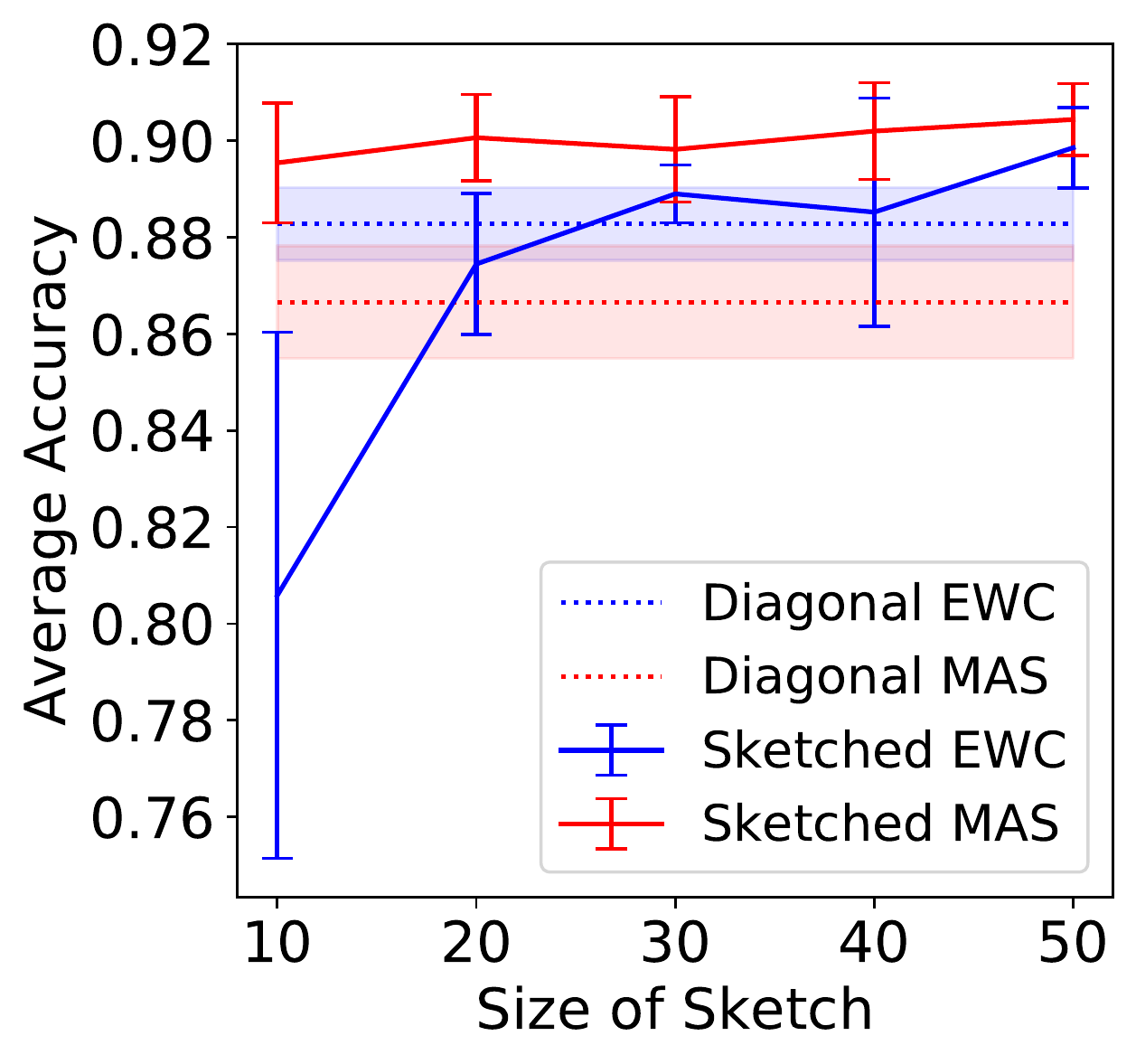}
            \caption{Effect of the sketch size ($t$) on the average accuracy of sketched methods for learning permuted-MNIST tasks. Two observations are made immediately: as the number of sketches increases, sketched methods tend to perform better; even using very small sketch size (e.g., $t\le 50$), sketched methods can outperform their diagonal counterparts.}
            \label{fig:perm_mnist_sketches_search}
        \end{subfigure}
    \end{minipage}%
\end{figure}

\paragraph{Effects of the Sketch Size.}
We then study the effects of the size of the sketch, i.e. $t$ in \eqref{eqn:sketched_regularizer}, on the performance of sketched SR.
The results are shown in Figure \ref{fig:perm_mnist_sketches_search}.
From the plot we see a clear trade-off between the size of the sketch and the average accuracy, where the average accuracy generally grows as the size of sketches increases --- however using more sketches costs more computation resources. 
Fortunately, even with a very small sketch size, e.g. $ t \ge 30 $, which is easily affordable in practice, sketched SR methods already significantly outperform diagonal SR methods. 
This demonstrates the practical effectiveness of the proposed sketched SR framework.

\begin{figure*}[t]
    \centering
    \includegraphics[width=1.0\linewidth]{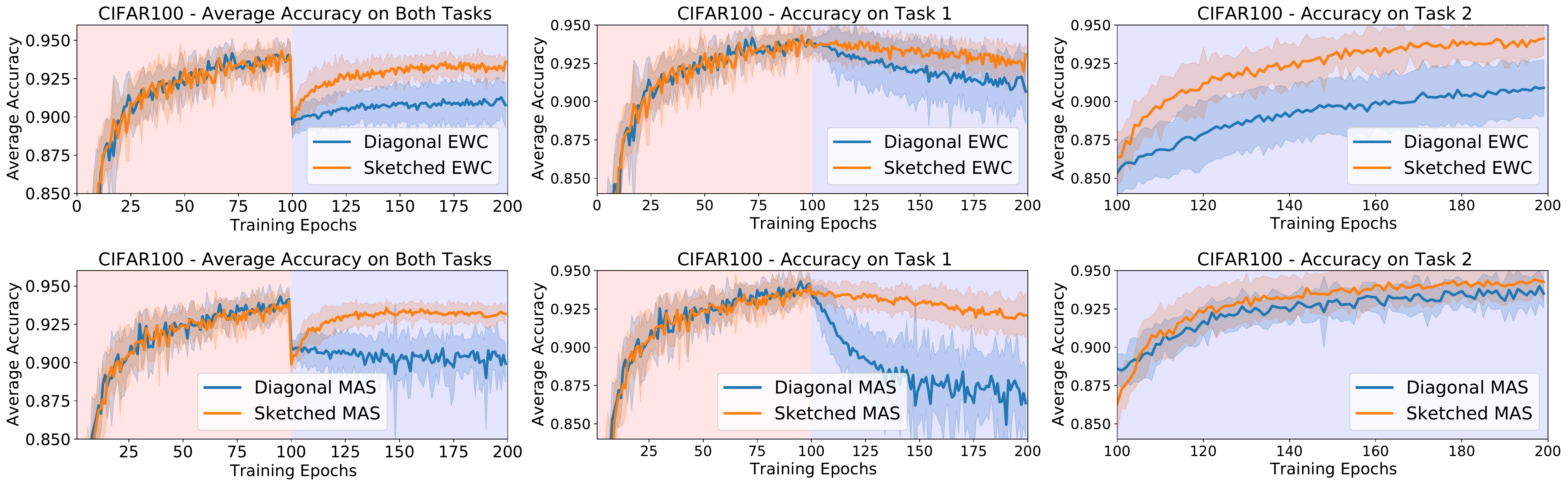}
    \caption{The average accuracy (over both tasks) of sketched SR and diagonal SR methods on CIFAR-100. 
    In each sub-figure, the first column shows the average test accuracy over both tasks and the ten repetitions, with corresponding standard deviations shown as shaded areas; the second column shows the average accuracy of task 1; and the third column shows it for task 2 (i.e., the first column is the average of the second and third columns).
    The plots suggest that, consistently, sketched variants are more effective than the diagonal versions in terms of overcoming catastrophic forgetting. See Section \ref{sec:cifar-exp} for more details.} 
    \label{fig:cifar100}
\end{figure*}

\subsection{CIFAR-100}\label{sec:cifar-exp}
Finally, we provide further verification for the effectiveness of our methods with CIFAR-100 experiments.

\paragraph{Setup.} 

We follow the \emph{CIFAR-100 Distribution Shift} task introduced in \cite{ramasesh2021anatomy}. The main difference from the split CIFAR experiment commonly used in the literature (see, e.g., \citep{zenke2017continual}) is that the \emph{CIFAR-100 Distribution Shift} does not require task-specific neural network heads for classifying classes of each task. Such a setting is consistent with our previous experiments, in which the same network is used to learn all tasks. 
In our experiment, similar to \cite{ramasesh2021anatomy}, both tasks are 5-class classification problems where each class is one of the 20 superclasses of the CIFAR-100 dataset. For instance, we take the five superclasses {\it aquatic mammals}, {\it fruits and vegetables}, {\it household electrical devices}, {\it trees}, and {\it vehicles-1}. The corresponding subclasses for Task 1 are (1) {\it dolphin}, (2) {\it apple}, (3) {\it lamp}, (4) {\it maple tree}, and (5) {\it bicycle}, while for Task 2, they are (1) {\it whale}, (2) {\it orange}, (3) {\it television}, (4) {\it willow}, and (5) {\it motorcycle}. 
In all experiments, we used a Wide-ResNet \citep{zagoruyko2016wide} as our backbone, and leveraged random flip, translation, and cutout \citep{devries2017improved} as augmentation.
We use ADAM with learning rate $10^{-3}$ as our optimizer for all experiments.
All reported results are averaged over $10$ runs with different random seeds.
For more details, please see Appendix \ref{sec:append-cifar}.

\paragraph{Performance of the Compared Algorithms.}
Figure \ref{fig:cifar100} shows the performance comparison between the sketched and diagonal variations of EWC and MAS methods. 
The plots suggest that sketched variants are significantly more effective than the diagonal versions in terms of overcoming catastrophic forgetting.
The results are consistent with those in synthetic experiments and permuted-MNIST experiments.

\begin{table}[ht]
\caption{The average accuracy (over all tasks) of sketched SR and diagonal SR methods on Permuted-MNIST and CIFAR-100. For sketched SR methods we set $t=50$.}
\label{table:mnist_cifar_result}
\vskip 0.15in
    \centering
    \begin{small}
    \begin{sc}
    \begin{tabular}{cc|cc}
        \toprule
        Dataset & Regime & Diagonal & Sketched\\
        \midrule
        \multirow{2}{*}{\makecell{Permuted-\\MNIST}} & EWC & 88.3$\pm$0.8\% & \textbf{89.8$\pm$0.9\%} \\
              & MAS & 86.7$\pm$1.2\% & \textbf{90.4$\pm$0.8\%} \\
        \midrule
        \multirow{2}{*}{\makecell{CIFAR-100}} & EWC & 90.8$\pm$1.5\% & \textbf{93.6$\pm$0.4\%} \\
              & MAS & 89.9$\pm$1.2\% & \textbf{93.2$\pm$0.6\%} \\
        \bottomrule
    \end{tabular}
    \end{sc}
    \end{small}
\vskip -0.1in
\end{table}

\section{Conclusion}\label{sec:conclusion}
In this paper we present sketched structural regularization as a general framework for overcoming catastrophic forgetting in lifelong learning.
Compared with the widely-used diagonal version of structural regularization approaches, our methods achieve better performance for overcoming catastrophic forgetting, since an improved approximation to the large importance matrix is adopted. 
In contrast to the inefficient low-rank approximation methods (e.g., PCA), the proposed sketched structural regularization is computational affordable for practical lifelong learning models.
Finally, the effectiveness of the proposed methods are verified in multiple benchmark lifelong learning tasks.

\bibliographystyle{plainnat}
\bibliography{ref}

\newpage
\onecolumn
\appendix

\section{Further Details on Theory}
\subsection{CountSketch Theory}\label{sec:append-theory}
The following theorem from \citet{cohen2016OSEStableRank} builds on several results for CountSketch matrices, giving theoretical guarantees for sketching quadratic forms of matrices. The theorem is re-phrased for our purposes, showing the quality of approximation by the sketch in preserving $\ell_2$-norms of vectors in the subspace spanned by the columns of $W$, the matrix that is being sketched. There is a trade-off in the quality of approximation by the sketch and its size, given by the dimension $t$ of the columns of the sketch matrix $S$. The quality of approximation depends on the $\ell_2$-norm of $\theta$ and the spectrum of $W$, namely the operator norm $\|W\|_2$ and the Frobenius norm $\|W\|_F$. 

\begin{thm}[Theorem 6, \citet{cohen2016OSEStableRank}]\label{thm:CountSketchThm_orig}
Let $W \in \mathbb{R}^{n \times m}$ be a matrix, $k \in \Nbb^+$ be a parameter and let $\epsilon, \delta > 0$ be constants. There exists a constant $C > 0$ such that a CountSketch matrix $S \in \mathbb{R}^{t \times n}$ with $t = \frac{Ck^2}{\epsilon^2\delta}$ has the property that for all $\theta \in \Rbb^m$,
\begin{align}\label{eqn:CountSketchGuarantee}
    {\abs{ \|SW \theta\|_2^2 - \|W \theta\|_2^2 }} \leq \epsilon \|\theta\|_2^2 \bracket{\|W\|_2^2 + \frac{\|W\|_F^2}{k}} 
\end{align}
with probability at least $1 -\delta$ and where the probability is taken over the randomness of the CountSketch matrix $S$.
\end{thm}
Theorem \ref{cor:sketch_guarantee} follows as a corollary of the above theorem by noting that when $t \geq \|W\|_F^4/(\epsilon^2\|W\|_2^4)$ the error scales with $\epsilon\|W\|_2^2\|\theta\|_2^2$.

\subsection{Proof of Theorem \ref{cor:online_guarantee}}\label{sec:append-online}
Throughout this section, we let $(a)_i$ denote the $i$-th entry of a vector $a \in \Rbb^m$ and let $(A)_j$ denote the $j$-th row of a matrix $A \in \Rbb^{n \times m}$.

We start with a lemma on the properties of matrix $S$ which we use in our proof of the theorem.
\begin{lem}\label{lemma:countSketch_properties}
The CountSketch matrix $S \in \Rbb^{t \times n}$ with sketch size $t \in \Nbb^+$ has the property that for any vector $y \in \mathbb{R}^n$ and index $i \in [t]$ and $j \neq i$, i) $\Ebb(Sy)_i=0$,  ii) $\Ebb[(Sy)_i(Sy)_j]=0$, and iii) $\Ebb(Sy)_i^2=\|y\|^2 / t$ where the expectation is taken over the randomness of the CountSketch matrix. 
\end{lem}
\begin{proof}
Let $S \in \Rbb^{t \times n}$ be the CountSketch matrix with sketch size $t$ resulting from the 2-wise independent hash function $h : [n] \rightarrow [t]$ and the 4-wise independent hash function $\sigma : [n] \times \{1, -1\}$ (see Algorithm \ref{alg:SSR} for descriptions of $h$ and $\sigma$). Let $y \in \Rbb^n$ be a vector and let $i, j \in [t]$ be indices such that $i \neq j$. 

To prove i), notice that $\expect{(Sy)_i} = \sum_{k = 1}^n \prob{h(k) = i} \cdot \expect{\sigma(k)} \cdot (y)_i = 0$ since $\expect{\sigma(k)} = 0$. 

To prove ii), we notice that by the definition of $h$, the random variable $\ind{h(k) = i}\ind{h(k) = j} = 0$ for any $i \neq j$. Then we can expand $\Ebb{(Sy)_i(Sy)_j}$ as follows:
\begin{align*}
    \expect{(Sy)_i(Sy)_j} &= 2\sum_{k = 2}^n \sum_{l = 1}^{k-1} \expect{\ind{h(k) = i}\ind{h(l) = j} \cdot \sigma(k)\sigma(l) \cdot (y)_i(y)_j} \\
    &=  2\sum_{k = 2}^n \sum_{l = 1}^{k-1} \expect{\sigma(k)\sigma(l)} \cdot \expect{\ind{h(k) = i}\ind{h(l) = j}(y)_i(y)_j } = 0 
\end{align*}
where the last equality follows from the fact that $\sigma$ is a 4-wise independent hash function and $i \neq j$.

Finally, we show property iii) as follows:
\begin{align*}
    \Ebb{(Sy)_i^2} &= \sum_{k = 1}^n \expect{\ind{h(k) = i}(y)_i^2} + 2\sum_{k = 2}^n \sum_{l = 1}^{k - 1} \expect{\sigma(k)\sigma(l)} \expect{\ind{h(k) = i}\ind{h(l) = i}(y)_i^2} \\ 
    &= \sum_{k = 1}^n \expect{\ind{h(k) = i}(y)_i^2} + 0 = \sum_{k = 1}^n \prob{\ind{h(k) = i}}(y)_i^2
    = \sum_{k = 1}^n \frac{1}{t} \cdot (y)_i^2 = \frac{\|y\|_2^2}{t}.
\end{align*}
In the second equality we used the fact that $\sigma$ is a 4-wise independent hash function.
\end{proof}
We are now ready to prove Theorem \ref{cor:online_guarantee}.
\begin{proof}[Proof of Theorem \ref{cor:online_guarantee}]Fix an arbitrary $\theta \in \Rbb^m$ and let $y_1, \dots, y_\tau \in \Rbb^n$ be the vectors such that $y_i = \sqrt{\alpha_i}W_i\theta$. We then have that: 
\begin{align*}
        \expect {\left\| \sum_{k = 1}^\tau S_k y_k \right\|^2 - \left(\sum_{k = 1}^\tau \|S_k y_k\|^2 \right)}
        &= \expect { 2\sum_{k = 2}^\tau \sum_{k<l} \sum_{i = 1}^t (S_l y_l)_i (S_k y_k)_i } \\
        &= 2\sum_{k<l} \sum_{i = 1}^t \Ebb{(S_l y_l)_i} \Ebb{(S_k y_k)_i} = 0.
\end{align*}
In the second equality we use the fact that $S_k$ and $S_l$ are independent random matrices for $l \neq k$ and property i) from Lemma \ref{lemma:countSketch_properties}. Next, we bound the variance:
\begingroup
\allowdisplaybreaks
\begin{align}
    & \var{\left\| \sum_{k = 1}^\tau S_k y_k \right\|^2 - \left(\sum_{k = 1}^\tau \|S_k y_k\|^2 \right)}
    = \expect { \left(  \left\| \sum_{k = 1}^\tau S_k y_k \right\|^2 - \left(\sum_{k = 1}^\tau \|S_k y_k\|^2 \right) \right)^2 } \notag  \\
    &= \expect { \left(2\sum_{k = 2}^\tau \sum_{l<k} \sum_{i = 1}^t (S_k y_k)_i (S_l y_l )_i \right)^2 } \notag \\
    &= 4 \underbrace{\expect {\sum_{k<l} \sum_i (S_k y_k )_i^2 (S_l y_l )_i^2}}_{z_1} + 4 \underbrace{\expect {\sum_{k < l} \sum_{i \neq j} (S_k y_k )_i (S_k y_k )_j (S_l y_l )_i (S_l y_l )_j}}_{z_2} \notag \\
    &+ 4 \underbrace{\expect {\sum_{\substack{k < l,  r < s \\ \text{ s.t. } \{k \neq r \text{ or } l \neq s \}} } \sum_{i,j} (S_k y_k )_i (S_l y_l )_i (S_r y_r )_j (S_s y_s )_j }}_{z_3}. \notag
\end{align}
We first argue that $z_3 = 0$; since either $k \neq r$ or $l \neq s$, without loss of generality let $k < r$. As a result, $k < l$ and $k<r<s$. We then have that $\Ebb[(S_k y_k )_i (S_l y_l )_i (S_r y_r )_j (S_s y_s )_j] = \Ebb[(S_k y_k )_i] \cdot \Ebb[(S_l y_l )_i (S_r y_r )_j (S_s y_s )_j] = 0$ since $\Ebb[(S_k y_k )_i] = 0$ using property i) from Lemma \ref{lemma:countSketch_properties}. 
Next we argue that $z_2 = 0$; since $\Ebb[(S_k y_k )_i (S_k y_k )_j] = 0$ using property ii) from Lemma \ref{lemma:countSketch_properties}, for any $k \in [\tau]$ and any $i \neq j$ we have that $\Ebb[(S_k y_k )_i (S_k y_k )_j (S_l y_l )_i (S_l y_l )_j] = \Ebb[(S_k y_k )_i (S_k y_k )_j] \cdot \Ebb[(S_l y_l )_i (S_l y_l )_j] = 0$. 

Finally, we can bound $z_1$ and hence the variance:
\begin{align*}
    z_1 &= 4 \sum_{k<l} \sum_i \Ebb(S_k y_k )_i^2 \Ebb(S_l y_l )_i^2  
    = 4 \sum_{k<l} t \frac{\|y_k \|^2}{t} \cdot \frac{\|y_l \|^2}{t} \\
    &= \frac{4}{t} \sum_{k<l} \|y_k \|^2 \cdot \|y_l \|^2 \leq \frac{2}{t} \left(\sum_{k = 1}^\tau \|y_k \|^2 \right)^2.
\end{align*}
\endgroup
In the second equality we use the property iii) from Lemma \ref{lemma:countSketch_properties} for $\Ebb(S_k y_k )_i^2$ and $\Ebb(S_l y_l )_i^2$.

The theorem follows by applying Chebyshev's inequality on $\| \sum_{k = 1}^\tau S_k y_k \|_2^2 - \sum_{k = 1}^\tau \|S_k y_k\|_2^2$ and the definition of $y_1, \dots, y_\tau$. 
\end{proof}

\section{Further Details on Experiments}
\subsection{Synthetic Experiments}\label{sec:append-toy}

\paragraph{Setups.} For the regularization matrix induced by EWC and MAS, we compare the performance of various approaches to approximating the importance matrix including: 
\begin{enumerate}[label=(\roman*),noitemsep,topsep=0mm,parsep=0mm,partopsep=0mm,leftmargin=*]
    \item a diagonal approximation;
    \item a block-diagonal approximation, with a sequence of $ 50 \times 50 $ non-zero blocks along the diagonal;
    \item Sketched SR with $t = 50$;
    \item a rank-1 SVD;
    \item a low rank (rank $ = 50 $) SVD;
    \item the full importance matrix.
\end{enumerate}
We use a small multi-layer perceptron with the architecture $ 2 \rightarrow 128 \rightarrow 64 \rightarrow 2 $ and with ReLU activation function. For all algorithms, we use ADAM as the optimizer with learning rate $ 10^{-3} $. The minibatch size is 100, and we use the importance parameter $ \lambda = 10^{3}$ and the online learning parameter $ \alpha = 0.5 $ for all experiments. We repeat all toy example experiments 5 times with different fixed seeds, and report the average accuracy on all tasks. These toy example experiments are conducted on one RTX2080Ti GPU. 

\paragraph{Online Learning in Synthetic Experiments.} For non-sketched approaches, the regularizer \eqref{eqn:regularizer_normVersion} in SR methods is approximated by
\begin{align}\label{eqn:OnlineLearningToy}
    \widetilde{\Rcal}(\theta) := \half (\theta - \theta^*_A)^\top \widetilde{\Omega} (\theta - \theta^*_A)
\end{align}
where $ \widetilde{\Omega} $ approximates the importance matrix $ \Omega $. The online extension of Sketched SR (see Section \ref{sec:SSR}) applies moving average on the sketch $ \widetilde{W} $, and cannot be directly applied on the regularizer in Equation \ref{eqn:OnlineLearningToy}. To ensure faithful comparison, moving average is applied on the importance matrix $ \widetilde{\Omega} $ in synthetic experiments according to Equation \eqref{eqn:omega_online_learning}.

\begin{figure}
    \centering
    \subfigure[EWC]{\includegraphics[width=0.45\linewidth]{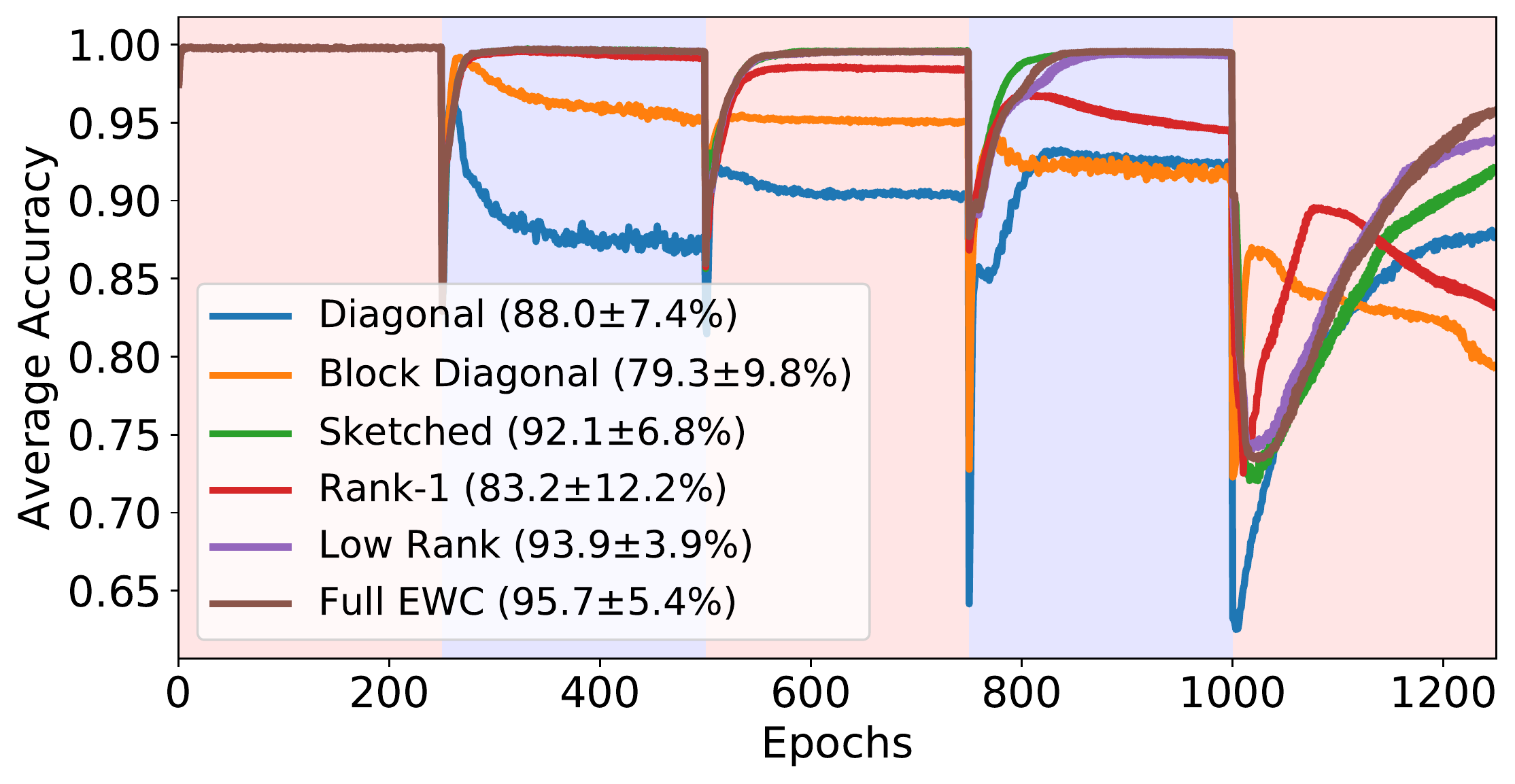}\label{fig:toy_ewc_accuracy}} 
    \subfigure[MAS]{\includegraphics[width=0.45\linewidth]{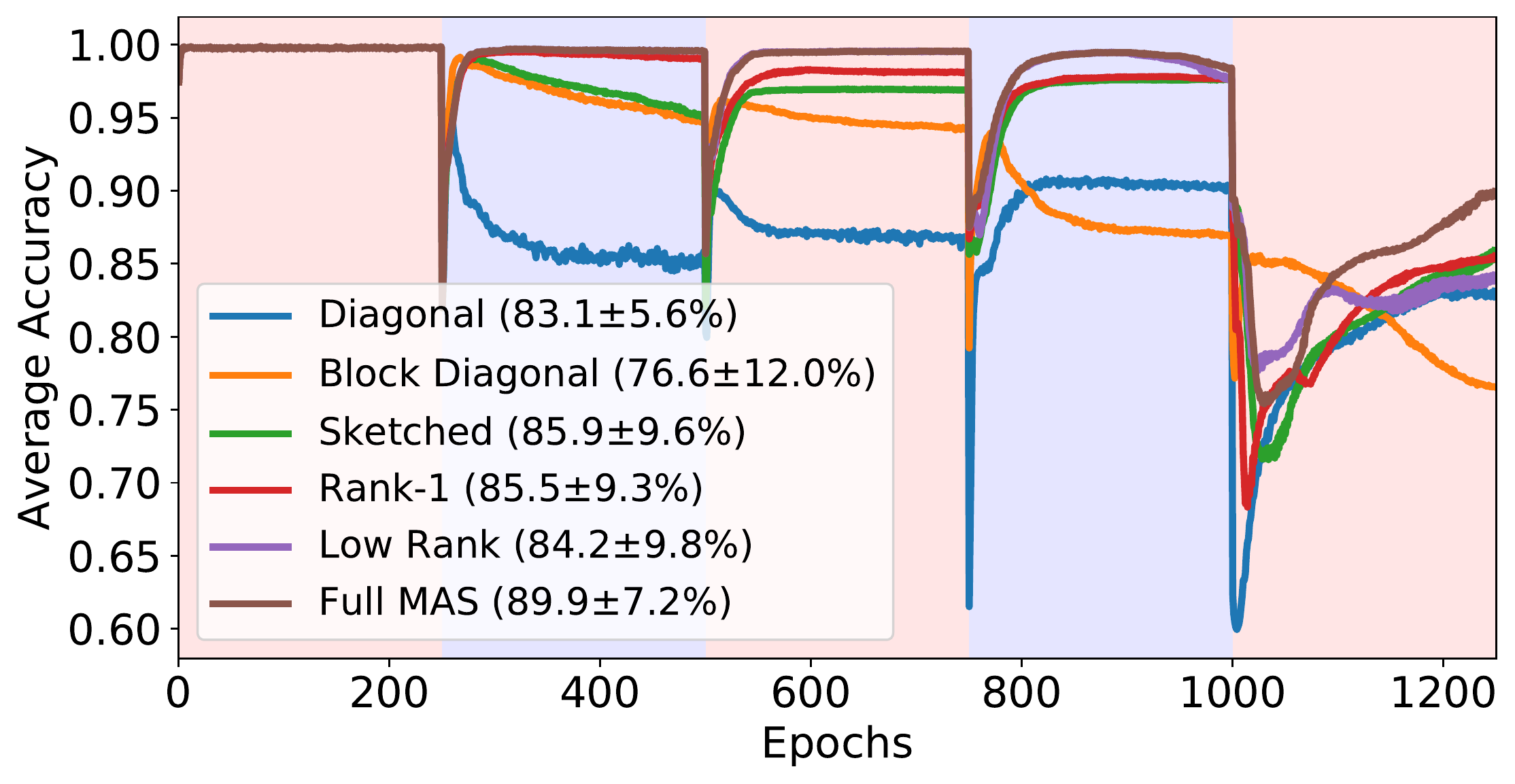}\label{fig:toy_mas_accuracy}} \\
    \caption{
    The average accuracy across previously learned tasks after each epoch of training, for all variants of SR on a synthetic 2D dataset from \citet{pan2020continual}.
    The number in the legend shows the average accuracy across all tasks (after learning the final task) for the compared algorithms.
    The plots suggest that: sketched SR outperforms both diagonal SR and block-diagonal SR for overcoming catastrophic forgetting; while rank-1 SR, low-rank SR and full SR are as good as or better than sketched SR in some cases, they requires significantly more computation which is not affordable in practice.}
    \label{fig:toy_accuracy}
\end{figure}

\begin{figure}[ht]
    \centering
    \subfigure{\includegraphics[width=0.15\linewidth]{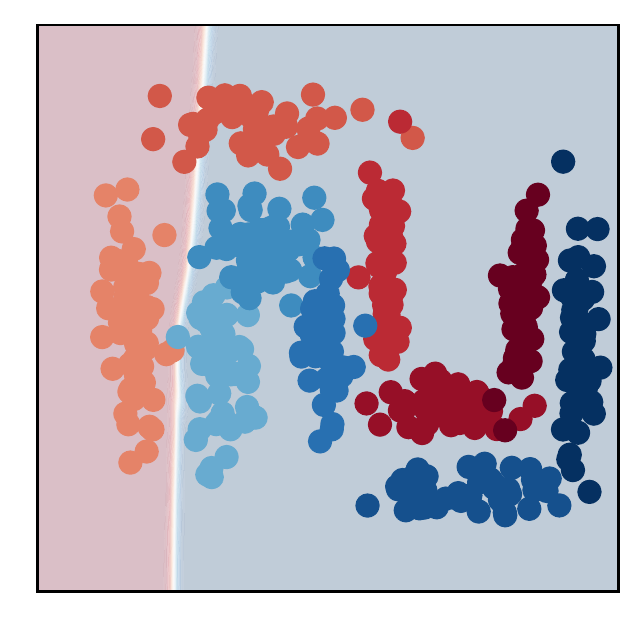}} 
    \subfigure{\includegraphics[width=0.15\linewidth]{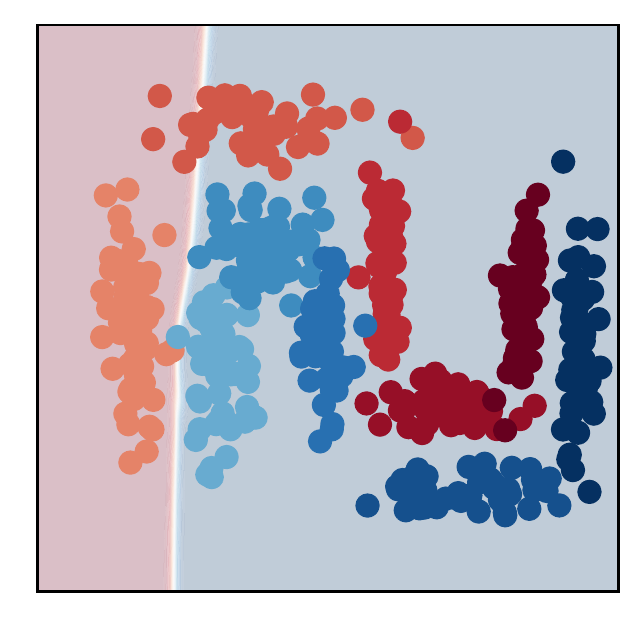}} 
    \subfigure{\includegraphics[width=0.15\linewidth]{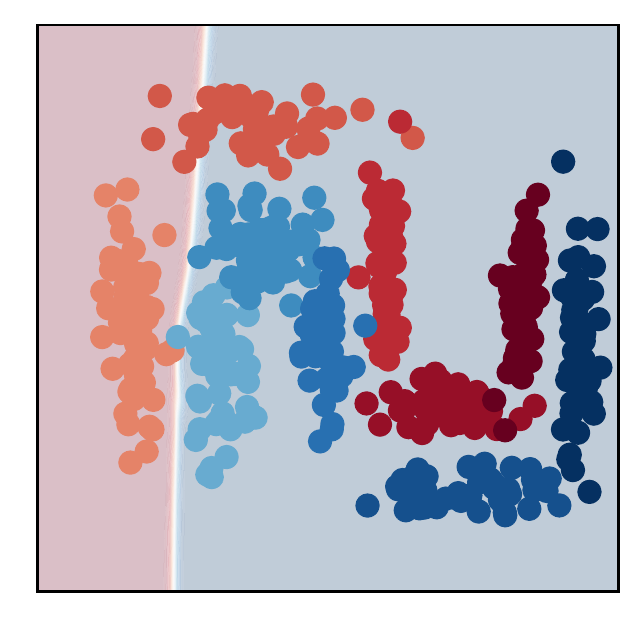}} 
    \subfigure{\includegraphics[width=0.15\linewidth]{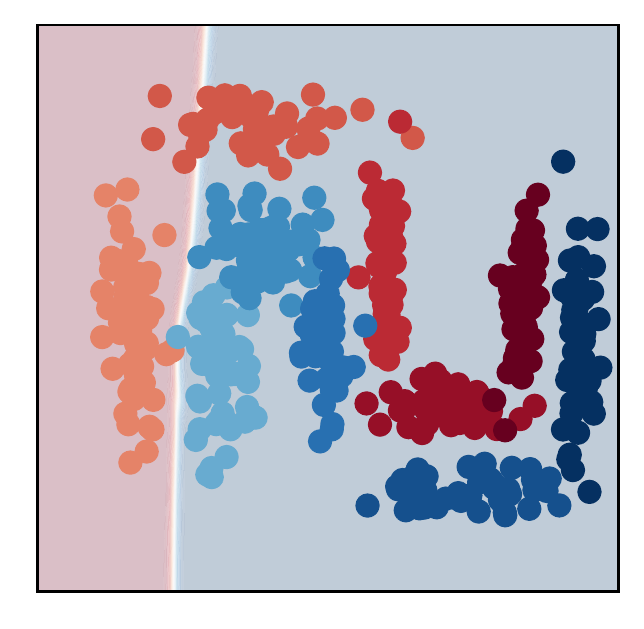}} 
    \subfigure{\includegraphics[width=0.15\linewidth]{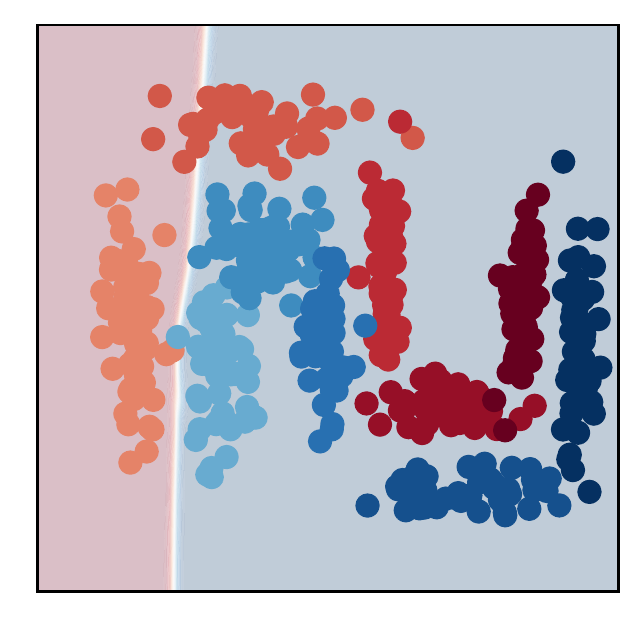}} 
    \subfigure{\includegraphics[width=0.15\linewidth]{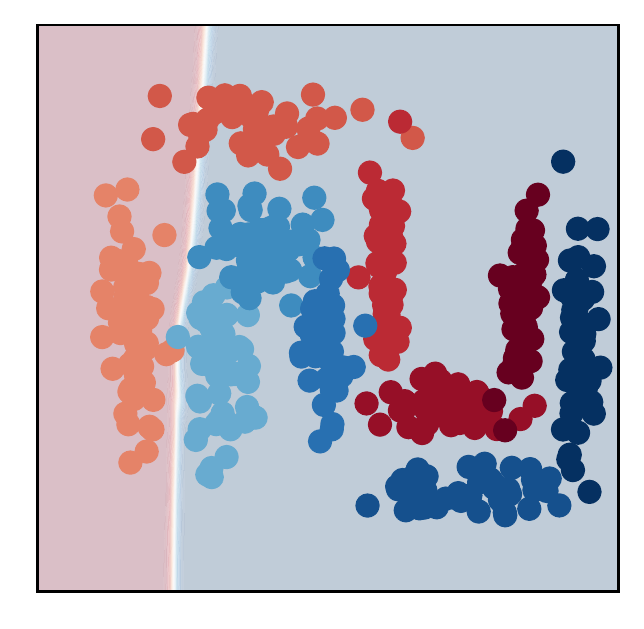}} \\
    \vspace{-1.25\baselineskip}
    
    \subfigure{\includegraphics[width=0.15\linewidth]{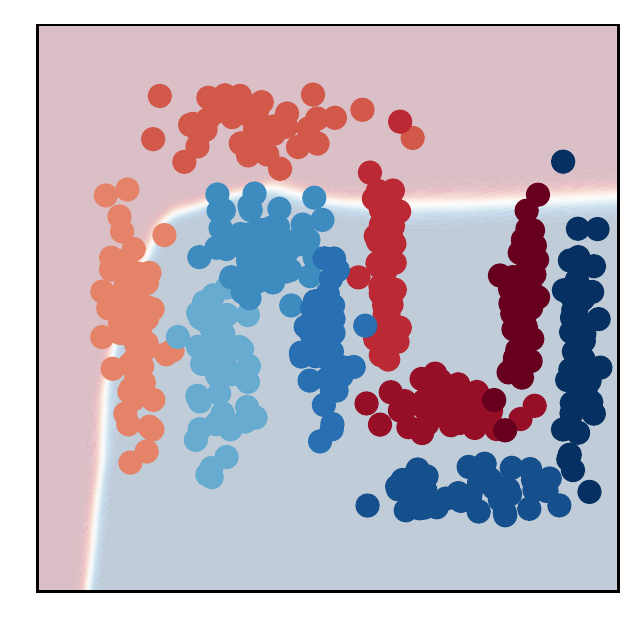}} 
    \subfigure{\includegraphics[width=0.15\linewidth]{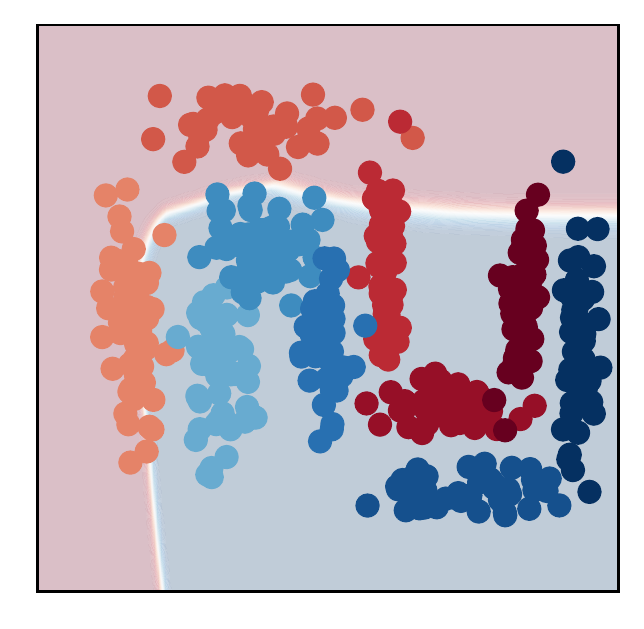}} 
    \subfigure{\includegraphics[width=0.15\linewidth]{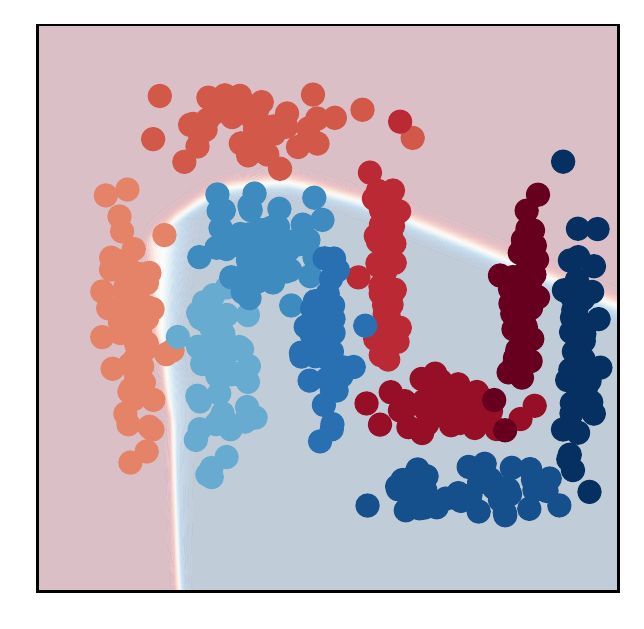}} 
    \subfigure{\includegraphics[width=0.15\linewidth]{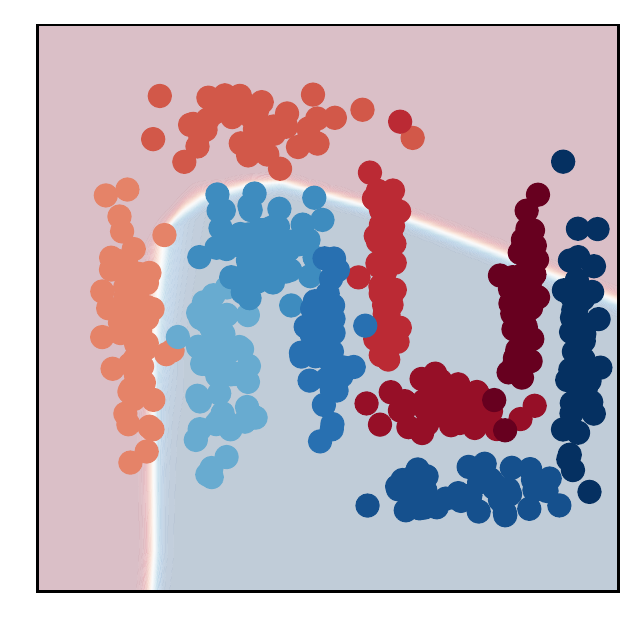}} 
    \subfigure{\includegraphics[width=0.15\linewidth]{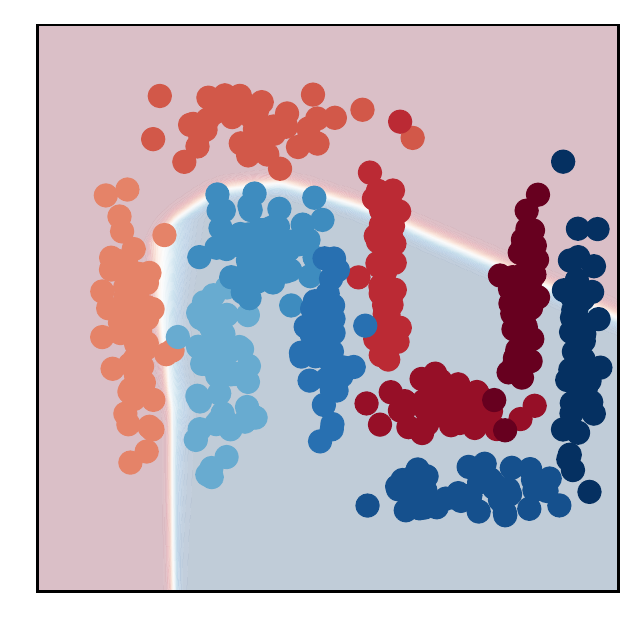}} 
    \subfigure{\includegraphics[width=0.15\linewidth]{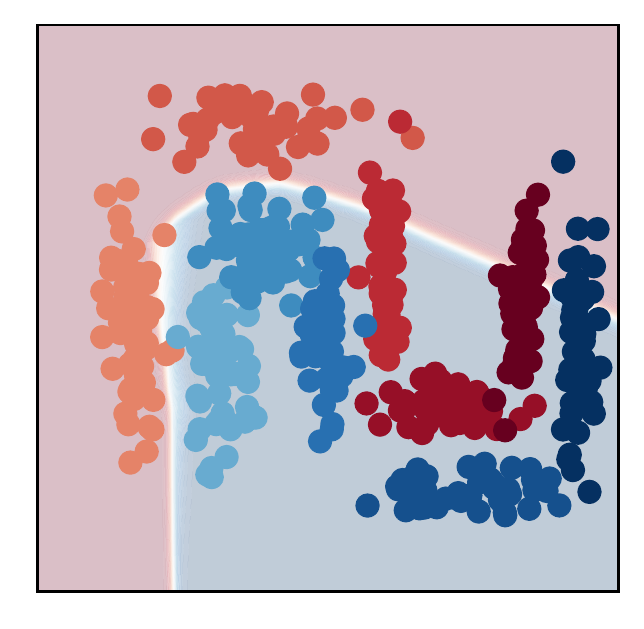}} \\
    \vspace{-1.25\baselineskip}
    
    \subfigure{\includegraphics[width=0.15\linewidth]{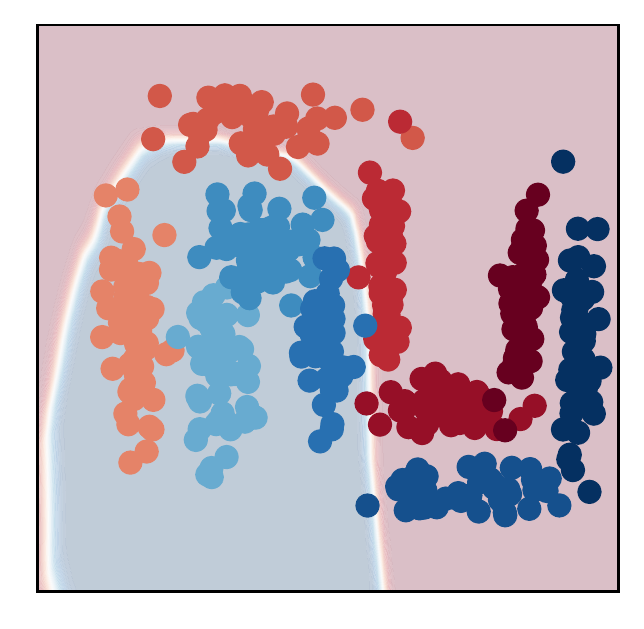}} 
    \subfigure{\includegraphics[width=0.15\linewidth]{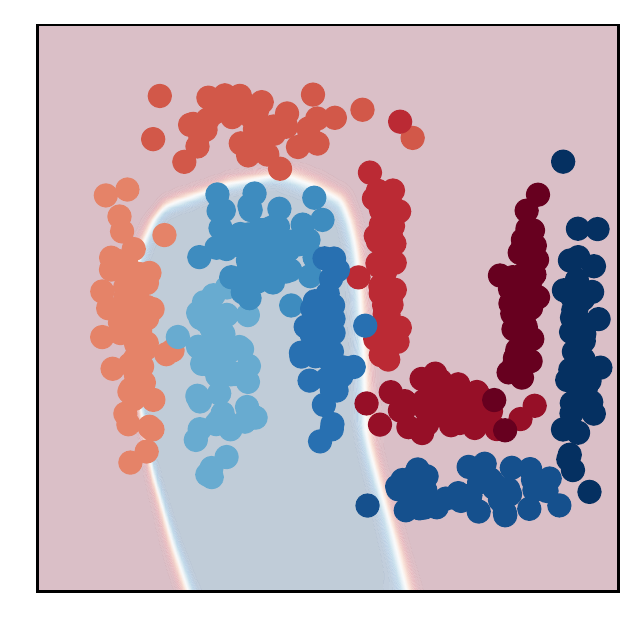}} 
    \subfigure{\includegraphics[width=0.15\linewidth]{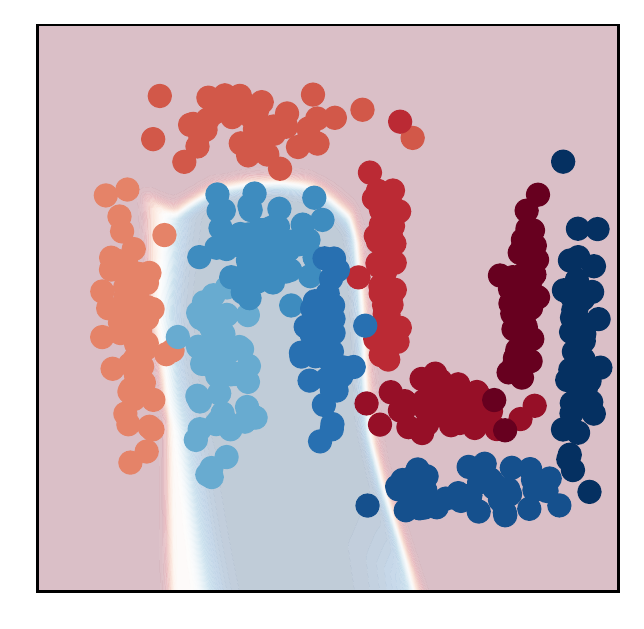}} 
    \subfigure{\includegraphics[width=0.15\linewidth]{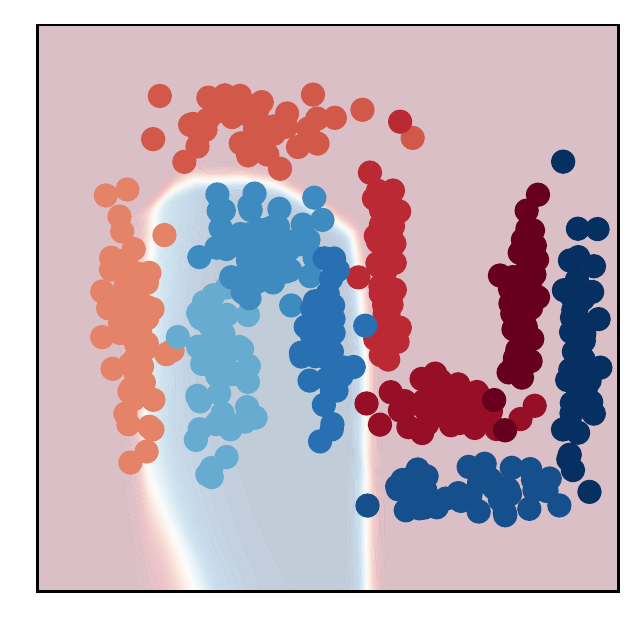}} 
    \subfigure{\includegraphics[width=0.15\linewidth]{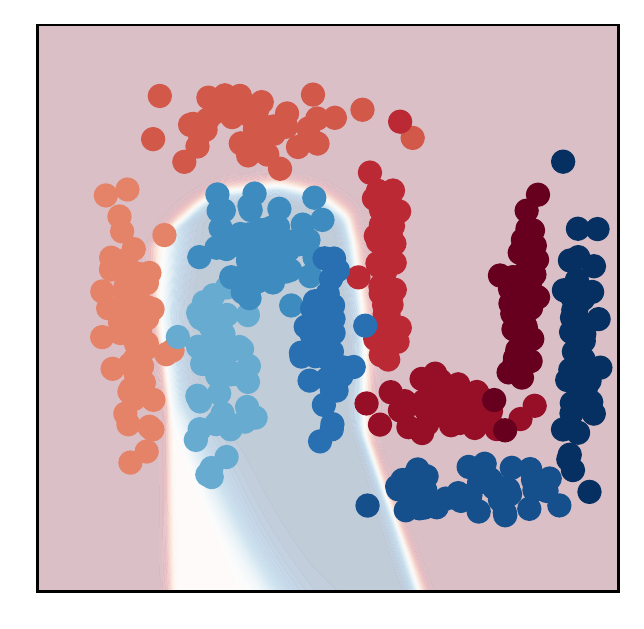}} 
    \subfigure{\includegraphics[width=0.15\linewidth]{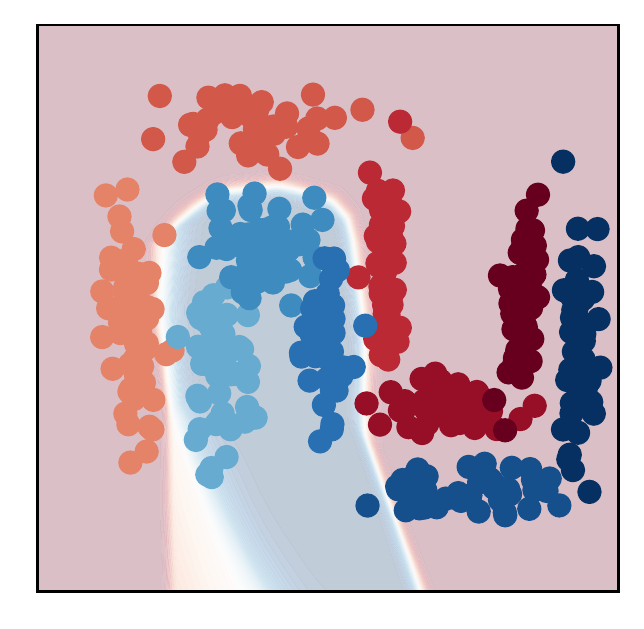}} \\
    \vspace{-1.25\baselineskip}
    
    \subfigure{\includegraphics[width=0.15\linewidth]{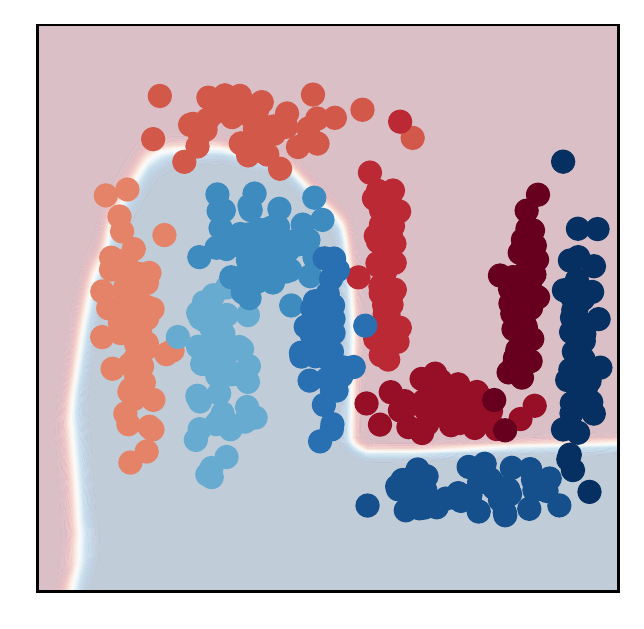}} 
    \subfigure{\includegraphics[width=0.15\linewidth]{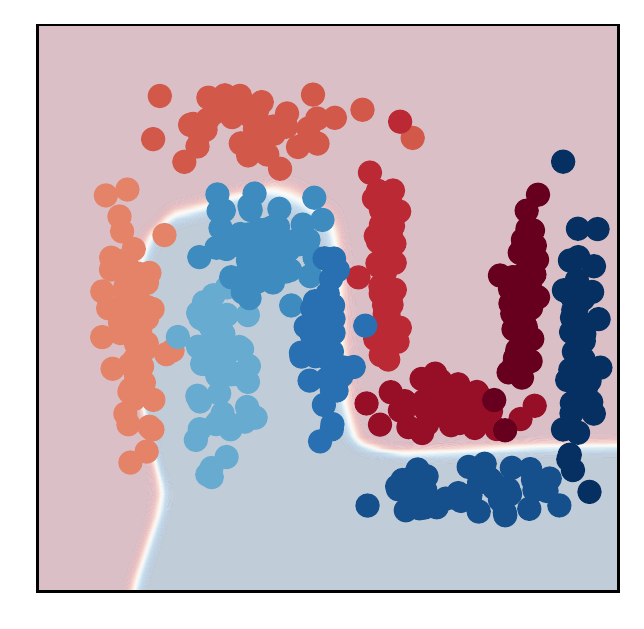}} 
    \subfigure{\includegraphics[width=0.15\linewidth]{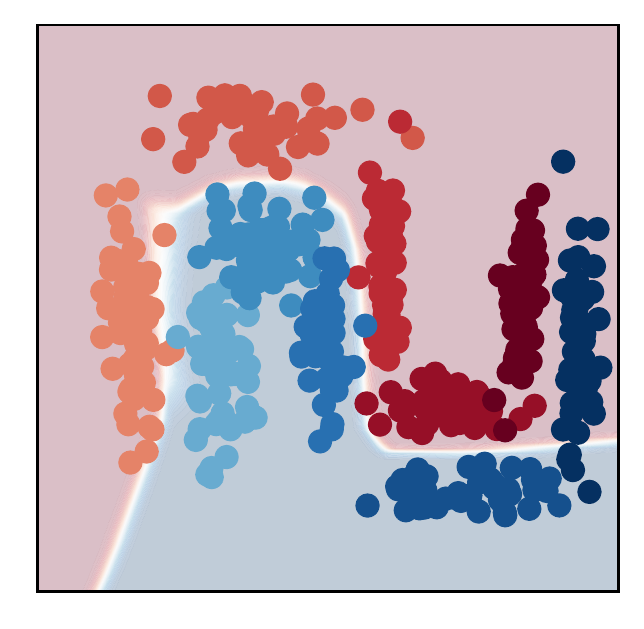}} 
    \subfigure{\includegraphics[width=0.15\linewidth]{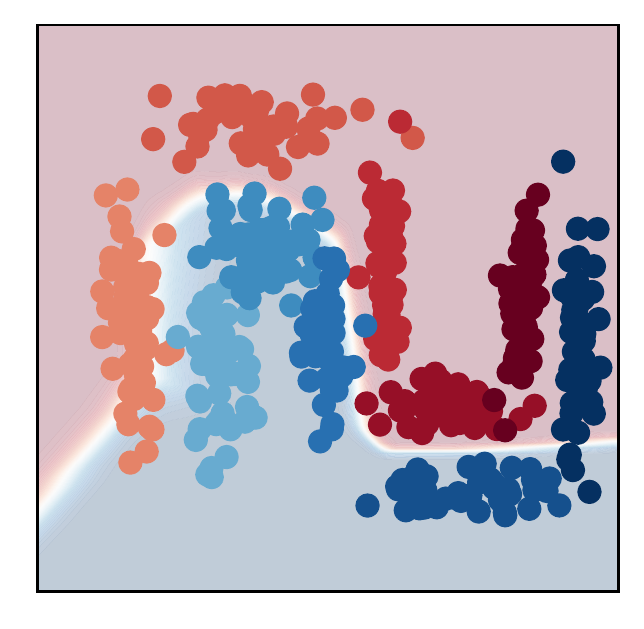}} 
    \subfigure{\includegraphics[width=0.15\linewidth]{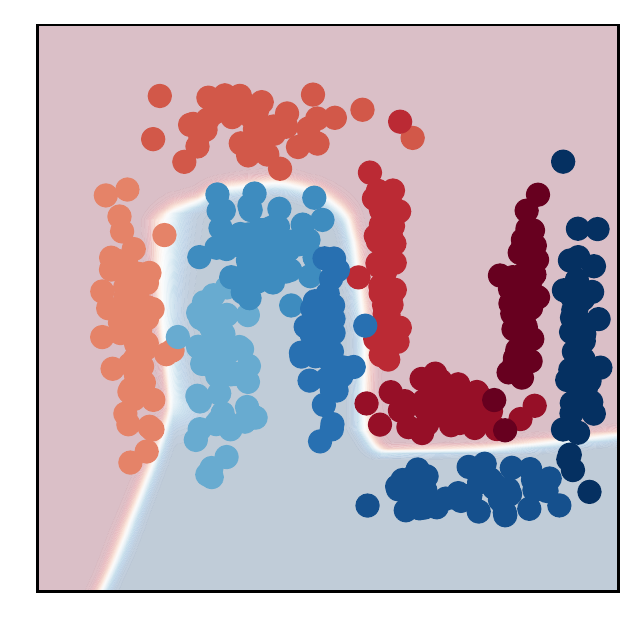}} 
    \subfigure{\includegraphics[width=0.15\linewidth]{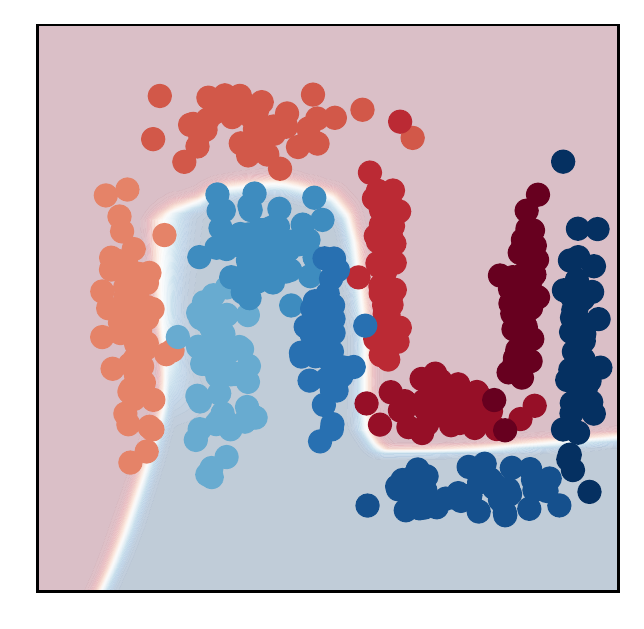}} \\
    \vspace{-1.25\baselineskip}
    
    \addtocounter{subfigure}{-24}
    \subfigure[Diagonal]{\includegraphics[width=0.15\linewidth]{figure/toy_ewc_task_5.pdf}} 
    \subfigure[Block-Diag.]{\includegraphics[width=0.15\linewidth]{figure/toy_block_diagonal_ewc_task_5.pdf}} 
    \subfigure[Sketched]{\includegraphics[width=0.15\linewidth]{figure/toy_sketch_ewc_task_5.pdf}} 
    \subfigure[Rank-1]{\includegraphics[width=0.15\linewidth]{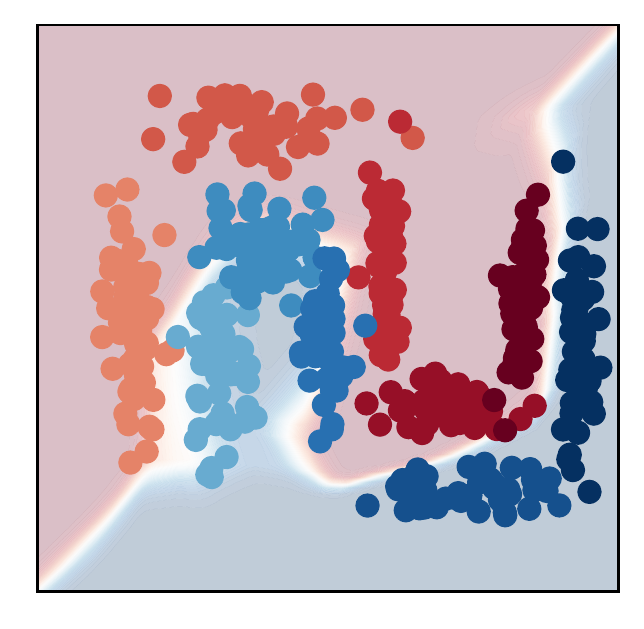}} 
    \subfigure[Low-Rank]{\includegraphics[width=0.15\linewidth]{figure/toy_low_rank_ewc_task_5.pdf}} 
    \subfigure[Full EWC]{\includegraphics[width=0.15\linewidth]{figure/toy_full_ewc_task_5.pdf}} \\
    \caption{
    Variants of EWC \citep{kirkpatrick2017overcoming} on a synthetic 2D binary classification dataset from \citet{pan2020continual}. 
    The two classes are represented by the different shades of red/blue, learnt sequentially using the variants of EWC.
    Each column of the figure shows the decision boundaries found by the algorithm (labelled below) after training each task, from the first task at the top to the last at the bottom.
    The plots suggest that: sketched EWC outperforms both diagonal EWC and block-diagonal EWC for overcoming catastrophic forgetting; while rank-1 EWC, low-rank EWC and full EWC are as good as or better than sketched EWC in some cases, they requires significantly more computation which is not affordable in practice.
    The observations are consistent with those in Figure \ref{fig:toy_accuracy}.}
    \label{fig:toy_ewc_illustration_full}
\end{figure}

\paragraph{Performance of the Compared Algorithms.}
Figure \ref{fig:toy_accuracy} reports the average accuracy across previously learned tasks after each epoch of training for the compared methods. The figures consistently show that sketched SR methods outperform their diagonal counterparts, in both EWC and MAS regimes, in terms of overcoming catastrophic forgetting. Figure \ref{fig:toy_ewc_illustration_full} and \ref{fig:toy_mas_illustration_full} further explore this, where decision boundary of the compared algorithms is plotted after training each task.
According to Figure \ref{fig:toy_ewc_illustration_full} and \ref{fig:toy_mas_illustration_full}, sketched SR methods forget less about the first task than diagonal SR, which directly demonstrate its advantage for overcoming catastrophic forgetting.
This corresponds to our observation in the permuted MNIST experiments.

\subsection{Permuted-MNIST}\label{sec:append-mnist}

\paragraph{Setup.}
We use a multi-layer perceptron with the architecture $ 784 \rightarrow 1024 \rightarrow 512 \rightarrow 256 \rightarrow 10 $ with ReLU activation function and no bias to learn this classification task. We use ADAM as the optimizer with learning rate $ 10^{-4} $ and the online learning parameter $ \alpha = 0.25 $ for all algorithms. The minibatch size is $100$. For each algorithm, a grid search on the regularization coefficient $ \lambda \in \{10^{i} \mid i = 2,3,\dots,6\} $ is used to determine the optimal hyperparameter for the reported results. 
We uses 50 sketches in Sketched SR to approximate the full importance matrix. All permuted-MNIST experiments are repeated $5$ times with different fixed seeds, and we report average accuracy on all tasks. We run permuted-MNIST experiments on a Tesla K80.

\paragraph{Effects of the Sketch Size per Task.}
We further study the effects of the size of the sketch $t$ (See Equation \ref{eqn:sketched_regularizer}) on the performance of sketched SR on each task.
The results are shown in Figure \ref{fig:perm_mnist_task_compare_2}.
From the plot we see a clear trade-off between the size of the sketch and the accuracy on later tasks, where the accuracy consistently increases as the size of sketches grows. This directly shows that increasing of the size of sketches improves learning capability for new tasks (known in the literature as \textit{intransigence}) with only little trade-off in catastrophic forgetting, with the expense of more computation resources. 

\begin{figure}[t]
    \centering
    \subfigure{\includegraphics[width=0.15\linewidth]{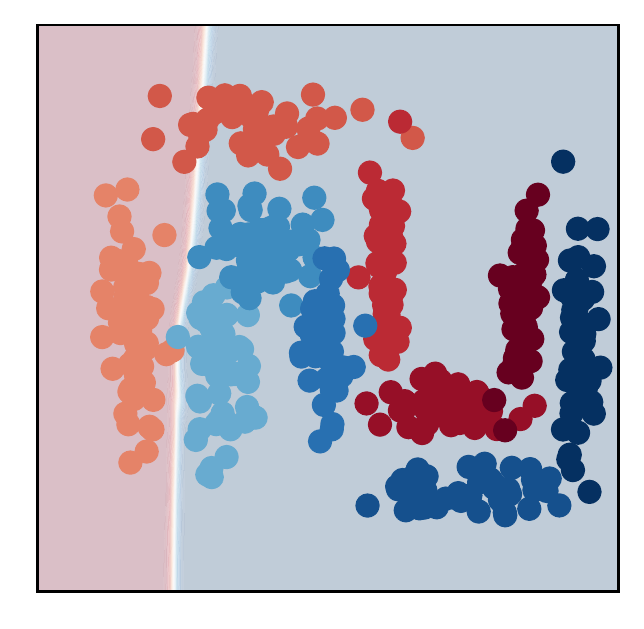}} 
    \subfigure{\includegraphics[width=0.15\linewidth]{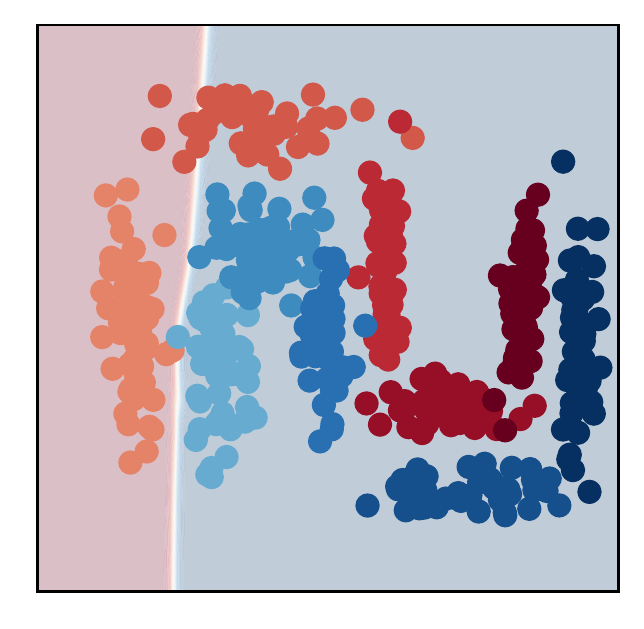}} 
    \subfigure{\includegraphics[width=0.15\linewidth]{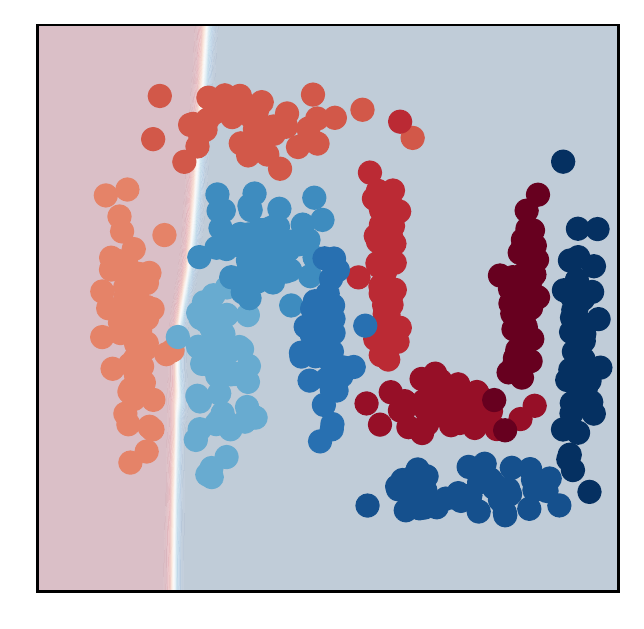}} 
    \subfigure{\includegraphics[width=0.15\linewidth]{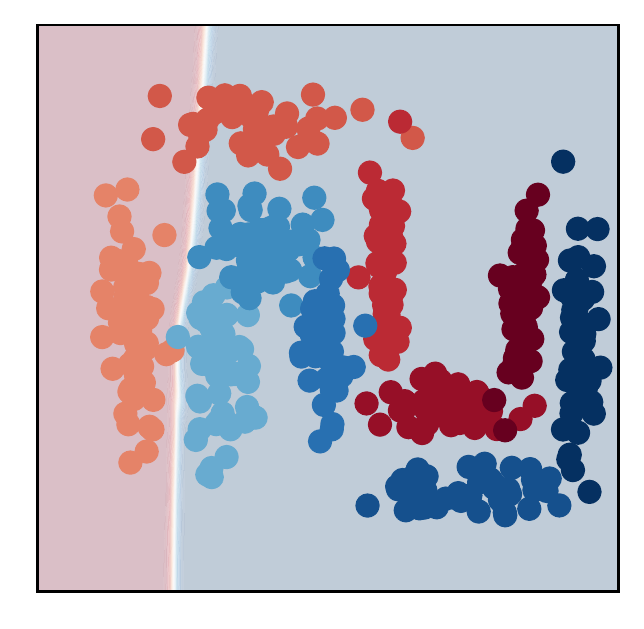}} 
    \subfigure{\includegraphics[width=0.15\linewidth]{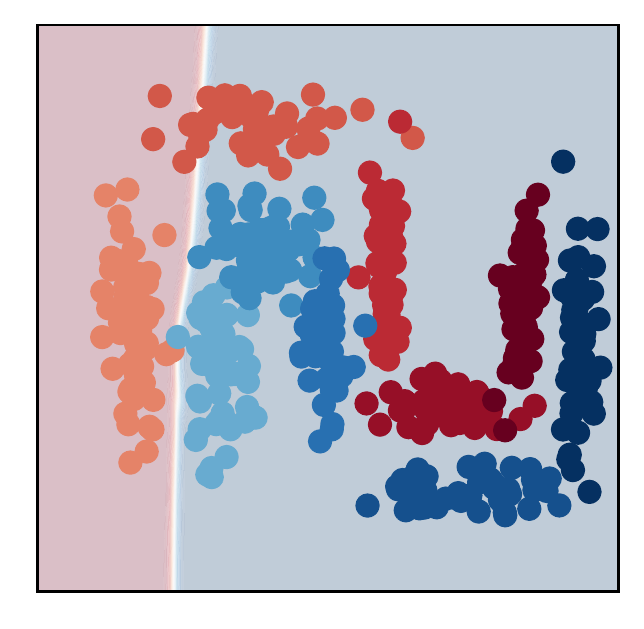}} 
    \subfigure{\includegraphics[width=0.15\linewidth]{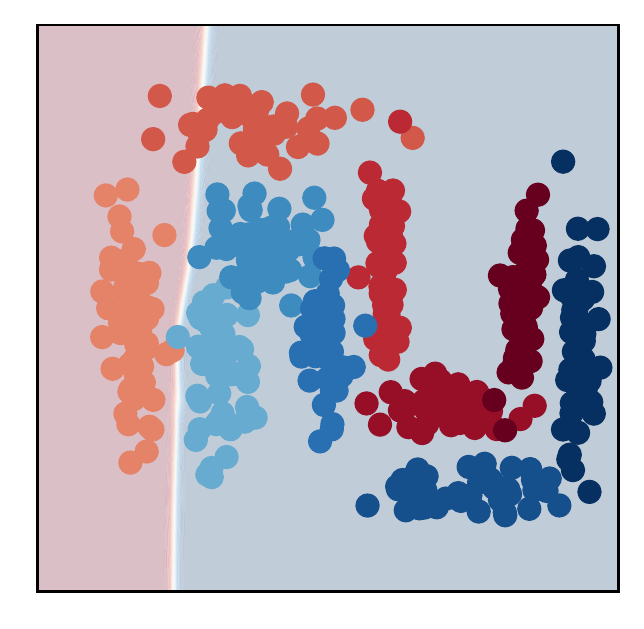}} \\
    \vspace{-1.25\baselineskip}
    
    \subfigure{\includegraphics[width=0.15\linewidth]{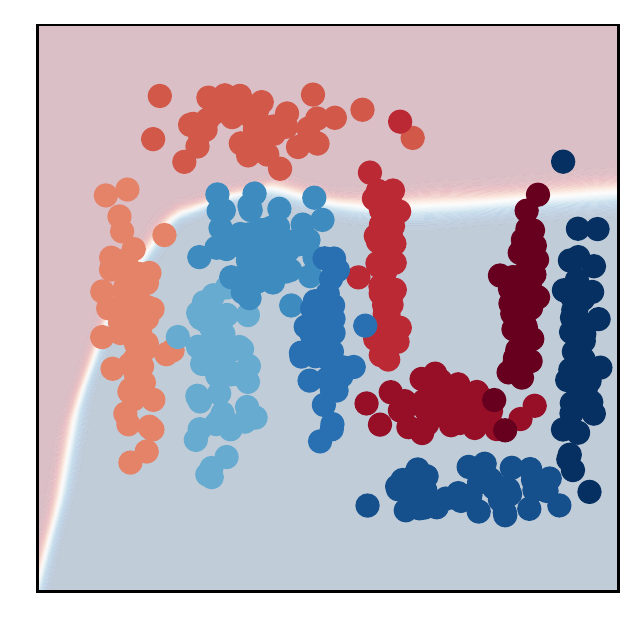}} 
    \subfigure{\includegraphics[width=0.15\linewidth]{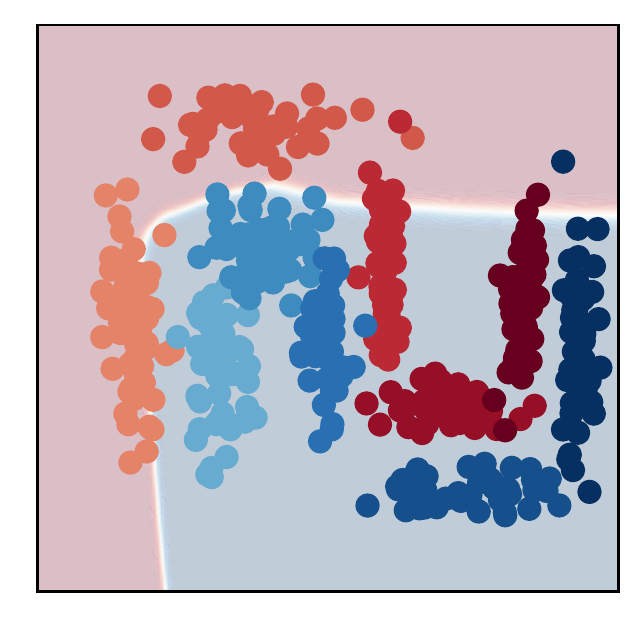}} 
    \subfigure{\includegraphics[width=0.15\linewidth]{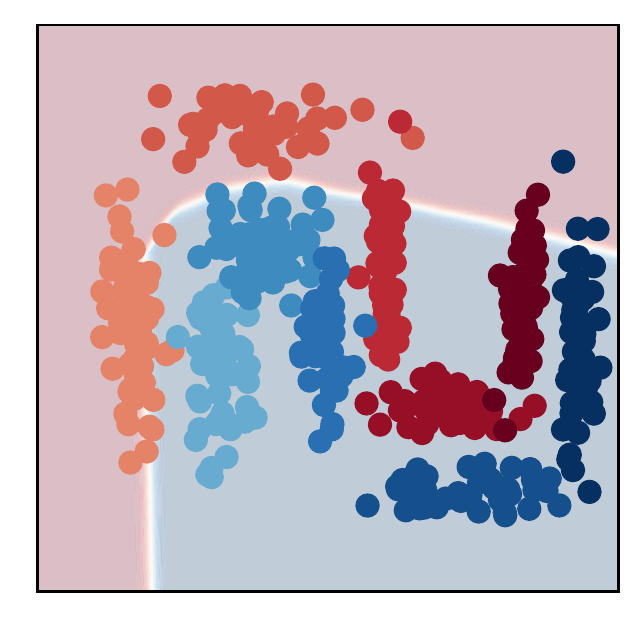}} 
    \subfigure{\includegraphics[width=0.15\linewidth]{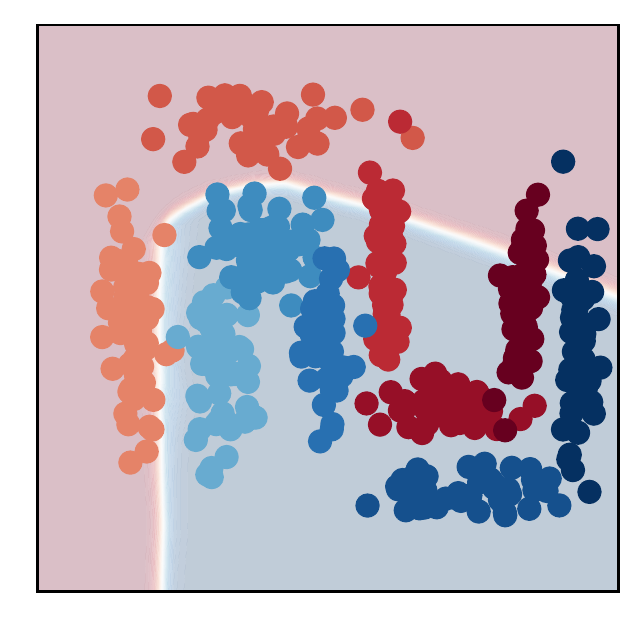}} 
    \subfigure{\includegraphics[width=0.15\linewidth]{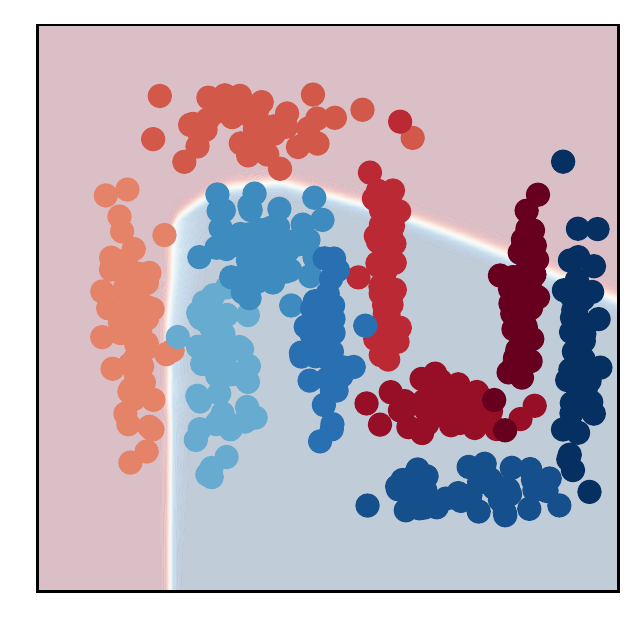}} 
    \subfigure{\includegraphics[width=0.15\linewidth]{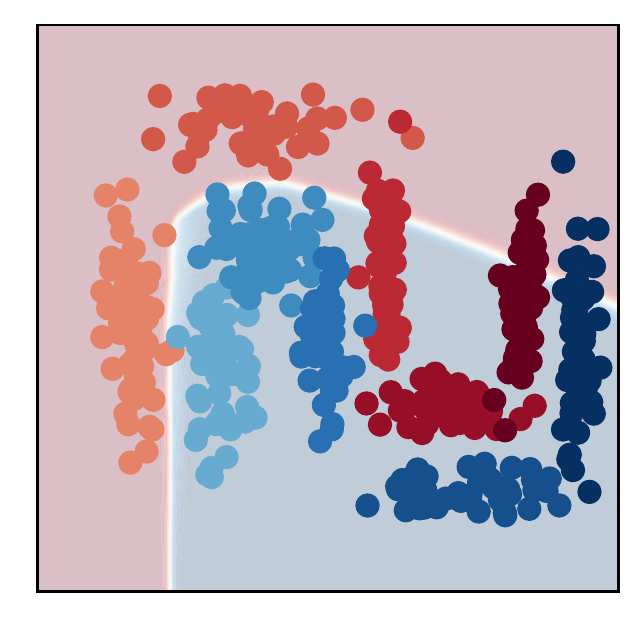}} \\
    \vspace{-1.25\baselineskip}
    
    \subfigure{\includegraphics[width=0.15\linewidth]{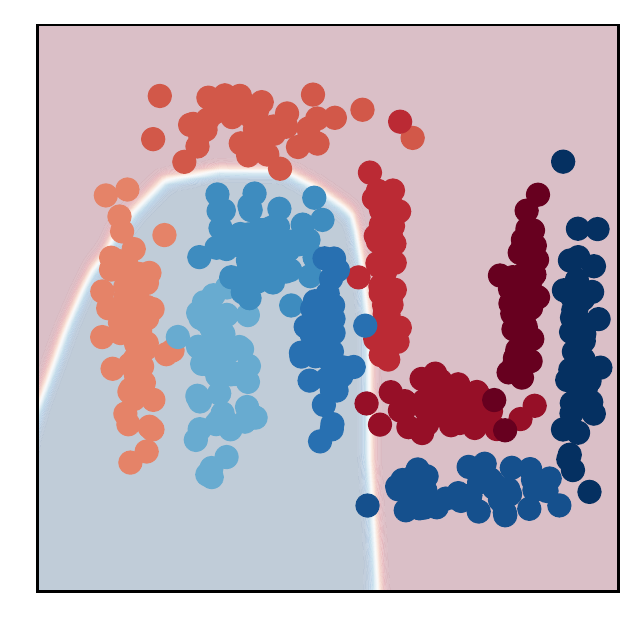}} 
    \subfigure{\includegraphics[width=0.15\linewidth]{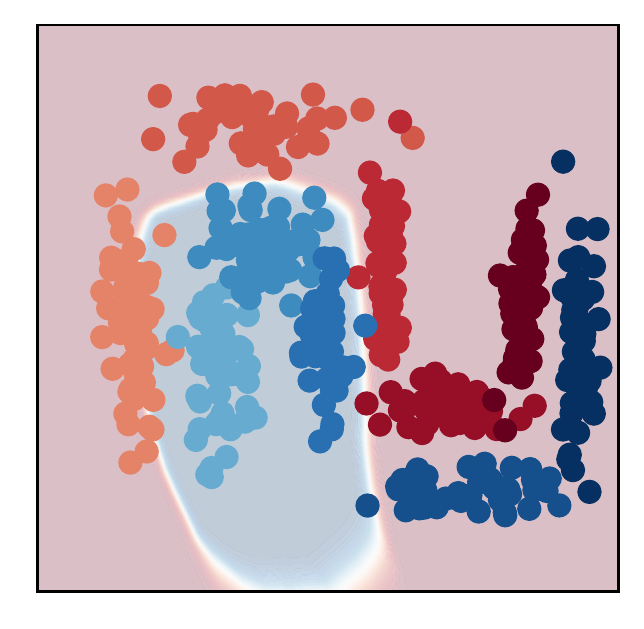}} 
    \subfigure{\includegraphics[width=0.15\linewidth]{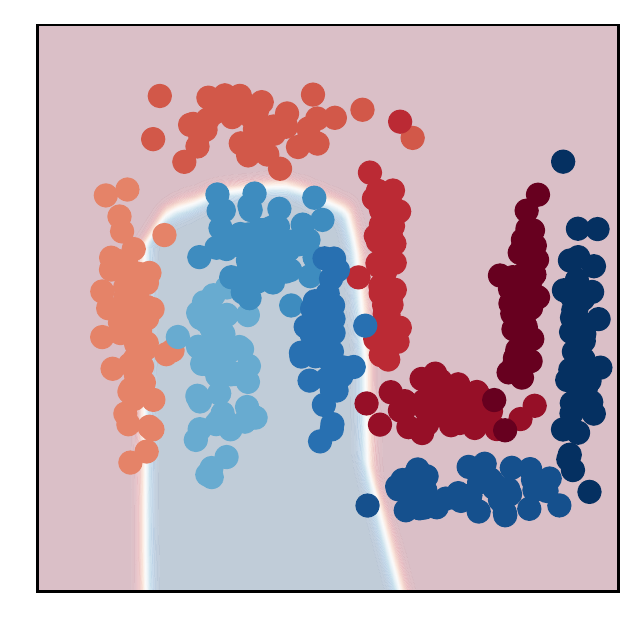}} 
    \subfigure{\includegraphics[width=0.15\linewidth]{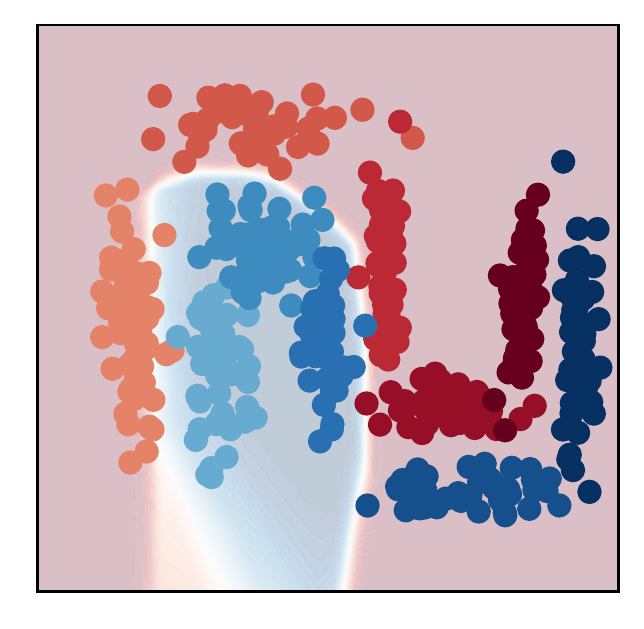}} 
    \subfigure{\includegraphics[width=0.15\linewidth]{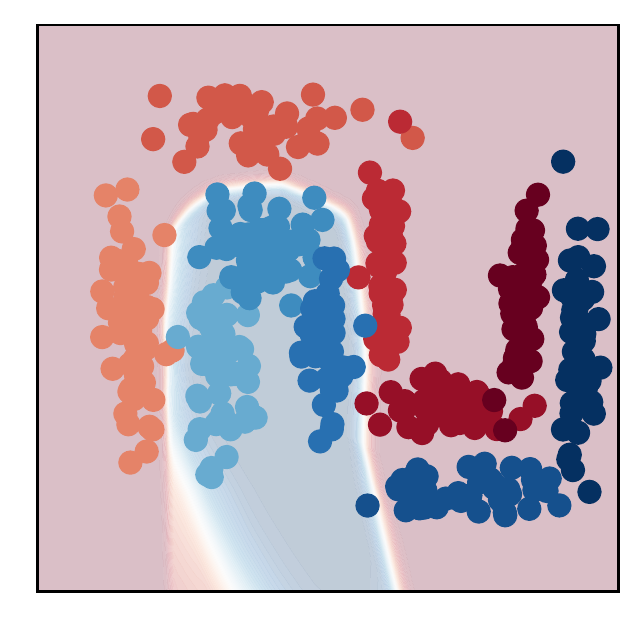}} 
    \subfigure{\includegraphics[width=0.15\linewidth]{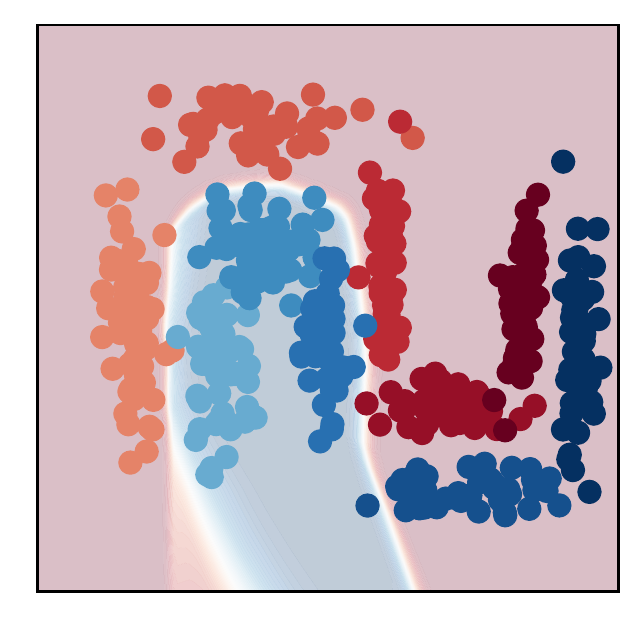}} \\
    \vspace{-1.25\baselineskip}
    
    \subfigure{\includegraphics[width=0.15\linewidth]{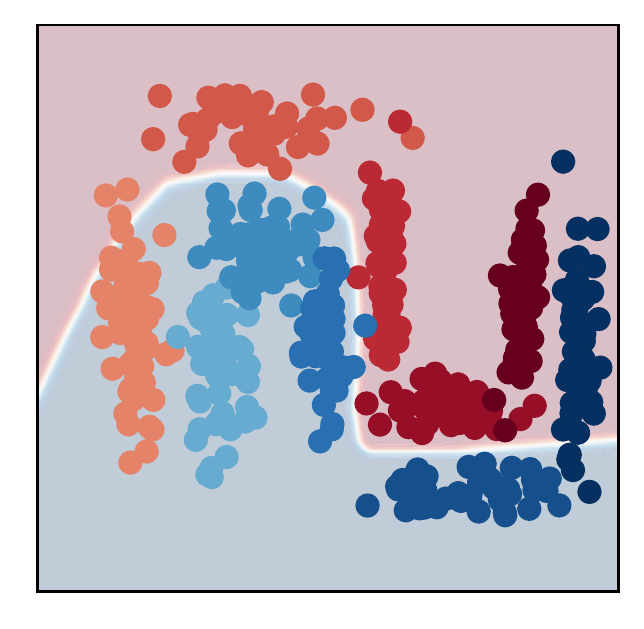}} 
    \subfigure{\includegraphics[width=0.15\linewidth]{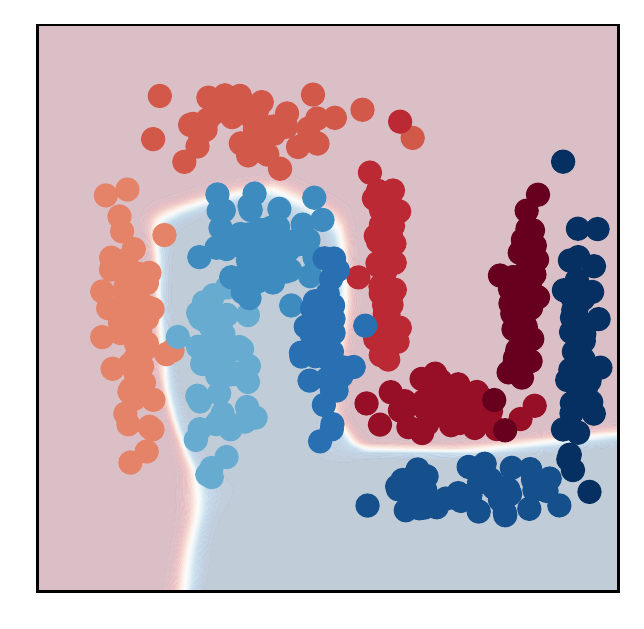}} 
    \subfigure{\includegraphics[width=0.15\linewidth]{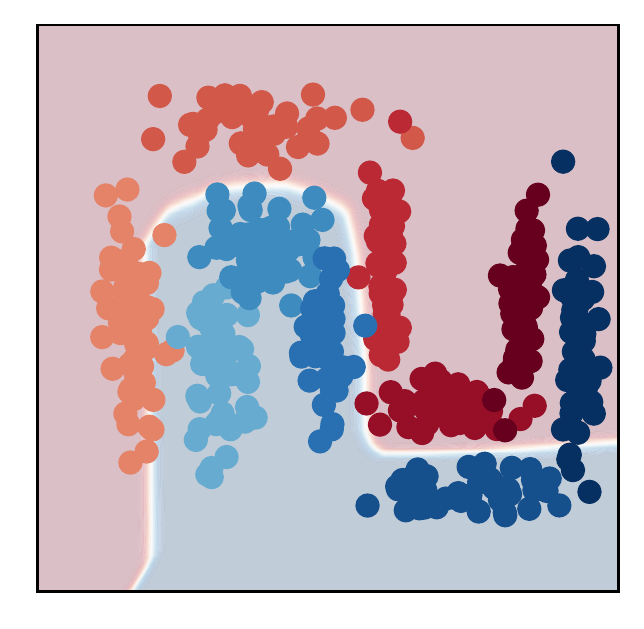}} 
    \subfigure{\includegraphics[width=0.15\linewidth]{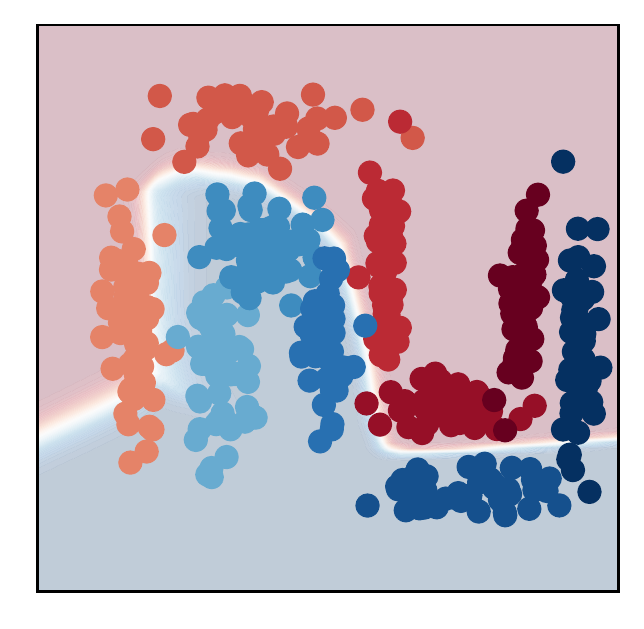}} 
    \subfigure{\includegraphics[width=0.15\linewidth]{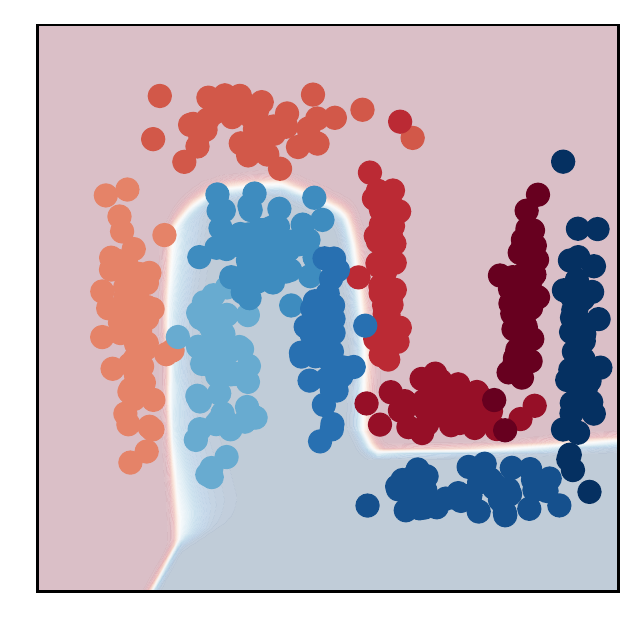}} 
    \subfigure{\includegraphics[width=0.15\linewidth]{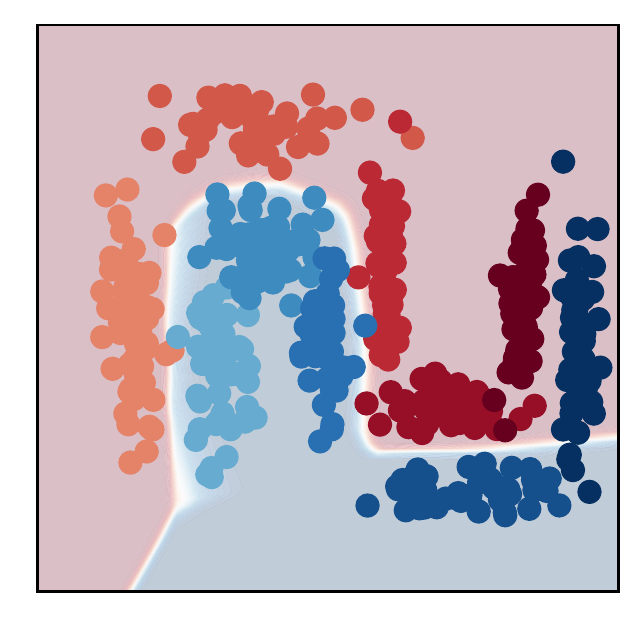}} \\
    \vspace{-1.25\baselineskip}
    
    \addtocounter{subfigure}{-24}
    \subfigure[Diagonal]{\includegraphics[width=0.15\linewidth]{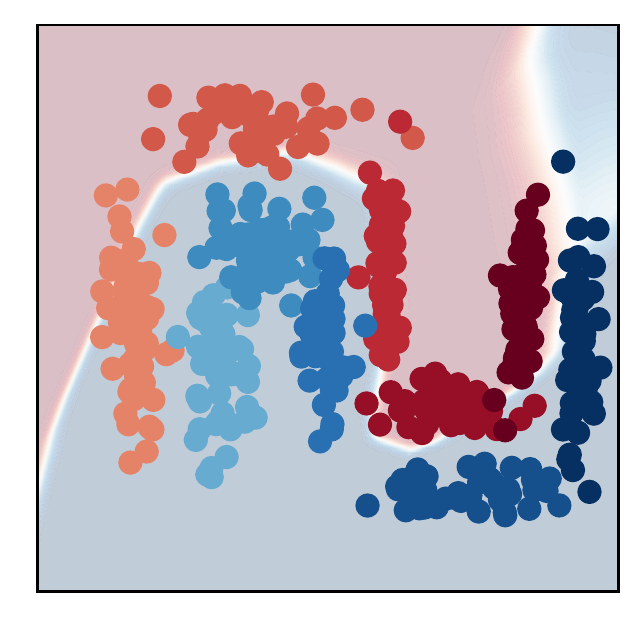}} 
    \subfigure[Block-Diag.]{\includegraphics[width=0.15\linewidth]{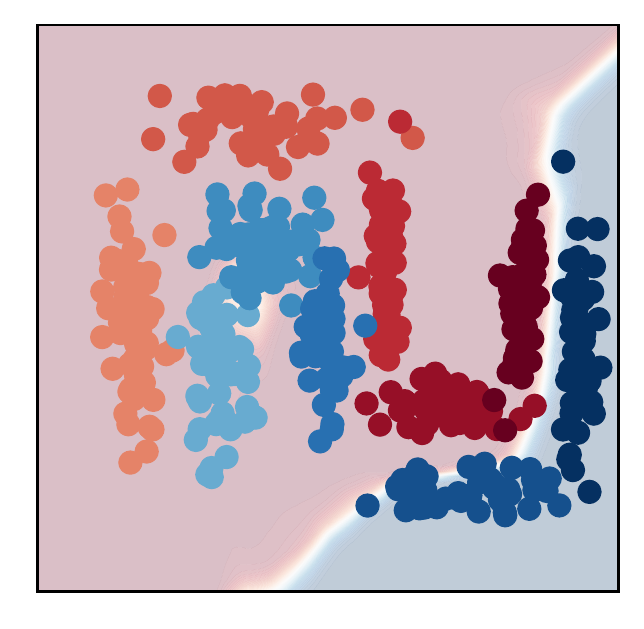}}
    \subfigure[Sketched]{\includegraphics[width=0.15\linewidth]{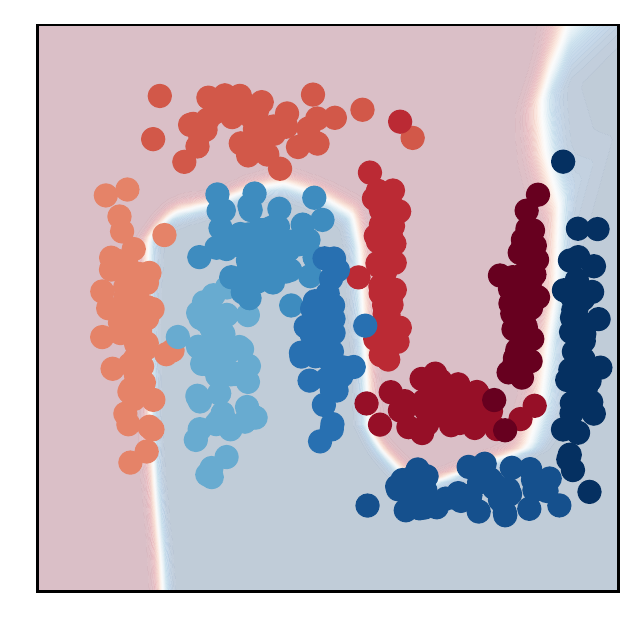}} 
    \subfigure[Rank-1]{\includegraphics[width=0.15\linewidth]{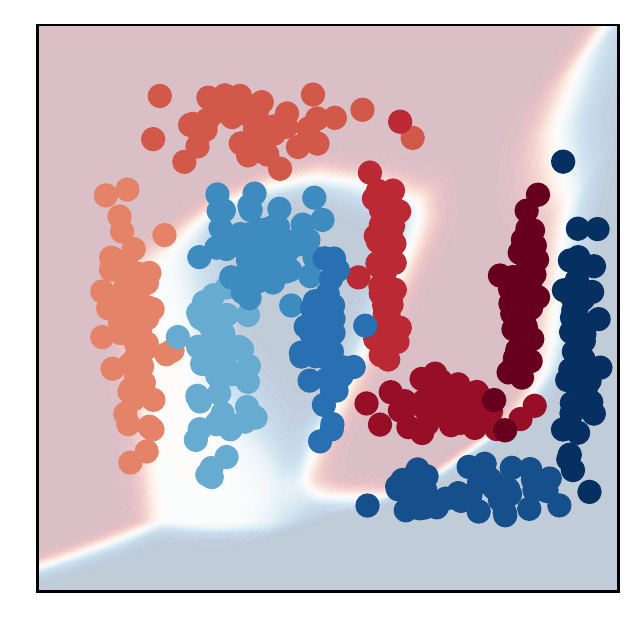}} 
    \subfigure[Low-Rank]{\includegraphics[width=0.15\linewidth]{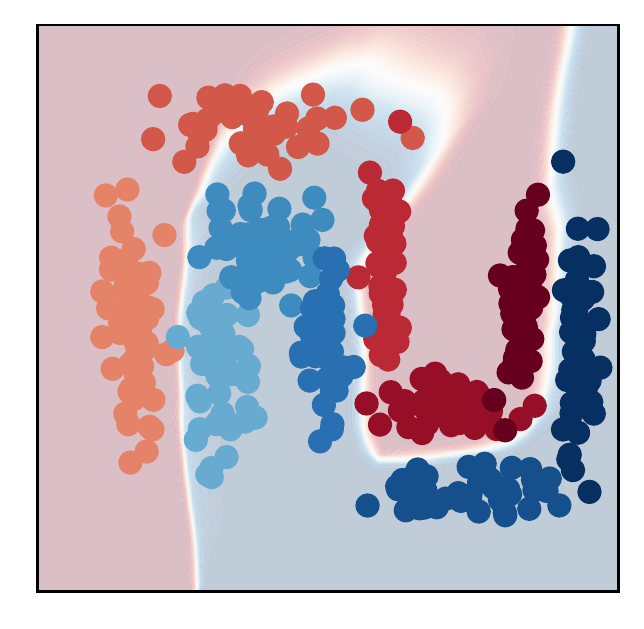}} 
    \subfigure[Full MAS]{\includegraphics[width=0.15\linewidth]{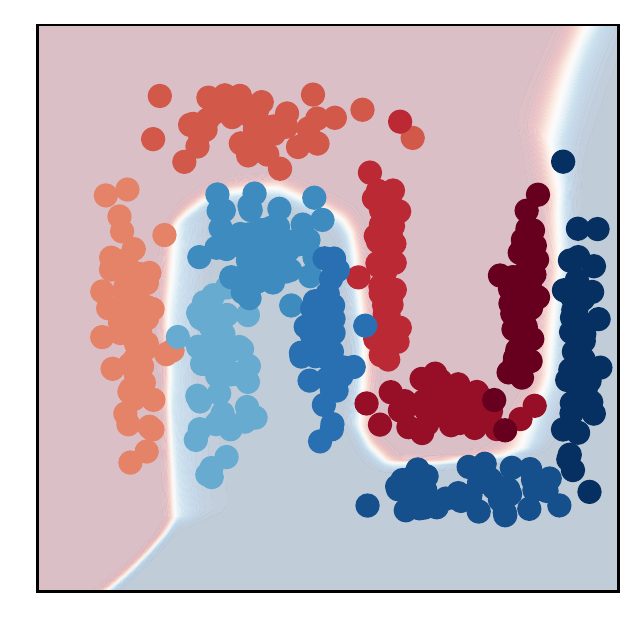}} \\
    \caption{Variants of MAS \citep{aljundi2018memory} on a synthetic 2D binary classification dataset from \citet{pan2020continual}. 
    The two classes are represented by the different shades of red/blue, learnt sequentially using the variants of MAS.
    Each column of the figure shows the decision boundaries found by the algorithm (labelled below) after training each task, from the first task at the top to the last at the bottom.
    The plots suggest that: sketched MAS outperforms both diagonal MAS and block-diagonal MAS for overcoming catastrophic forgetting; while rank-1 MAS, low-rank MAS and full MAS are as good as or better than sketched MAS in some cases, they requires significantly more computation which is not affordable in practice.
    The observations are consistent with those in Figure \ref{fig:toy_accuracy} and \ref{fig:toy_ewc_illustration_full}.}
    \label{fig:toy_mas_illustration_full}
\end{figure}

\begin{figure}[H]
    \centering
    \subfigure[Sketched EWC]{\includegraphics[width=0.3\linewidth]{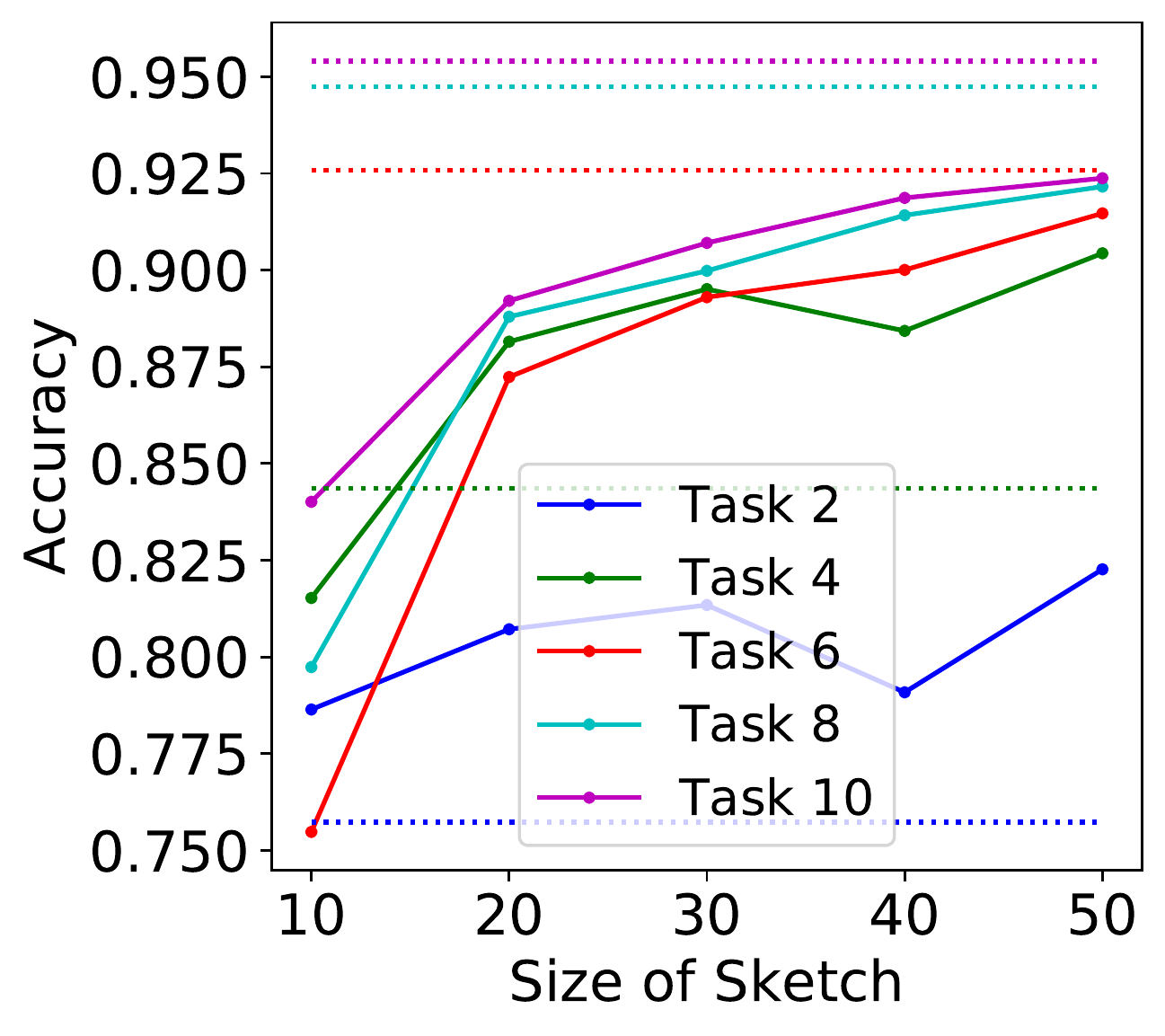}\label{fig:perm_mnist_ewc_task_comparison_2}}
    \subfigure[Sketched MAS]{\includegraphics[width=0.3\linewidth]{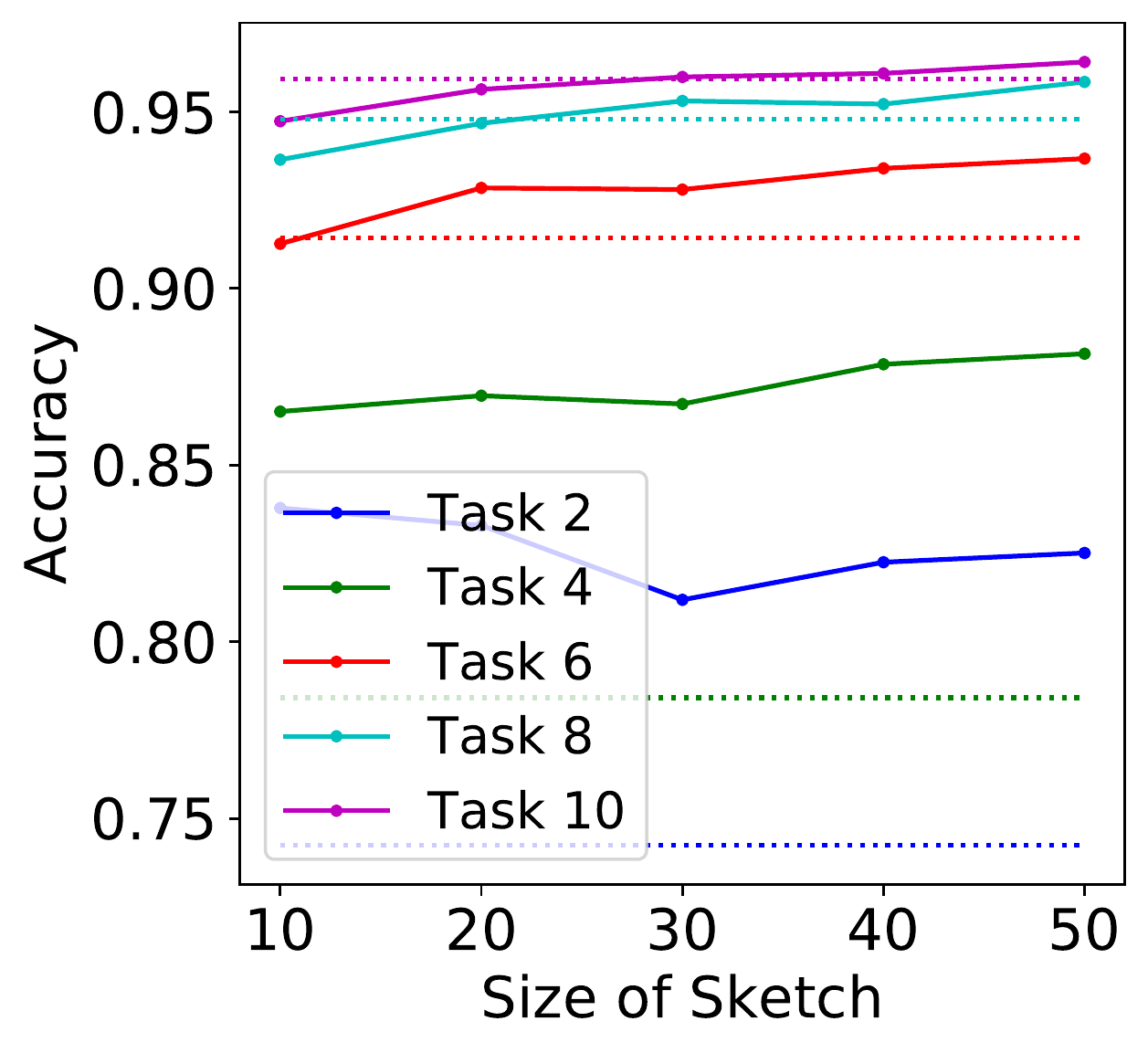}\label{fig:perm_mnist_mas_task_comparison_2}}
    \caption{Effect of the sketch size ($t$) on task accuracy of sketched methods for learning 10 permuted-MNIST tasks. Dotted line represents the accuracy of diagonal methods on the corresponding task with the same color. We can immediately observe that as the number of sketches increases, sketched methods tend to perform better in later tasks (Task ID $ \geq 6 $). 
    }
    \label{fig:perm_mnist_task_compare_2}
\end{figure}

\subsection{CIFAR-100}\label{sec:append-cifar}

\begin{figure}[htbp]
    \centering
    \includegraphics[width=\linewidth]{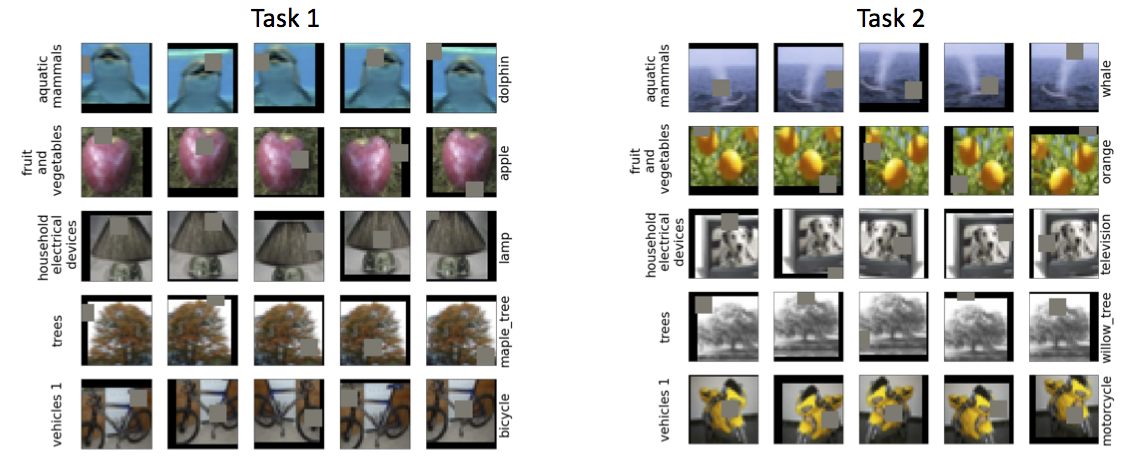}
    \vspace{-1.25\baselineskip}
    \caption{Sample images with 5 random augmentations for Task 1 and Task 2 in our CIFAR-100 experiment. The five superclasses for both tasks are represented by each row (labelled on the left), while the corresponding subclasses for each task are represented by rows within the task (labelled on the right).}
    \label{fig:augmentation}
\end{figure}

\paragraph{Dataset.}
For our CIFAR-100 experiment, we follow the 2-task \emph{CIFAR-100 Distribution Shift} dataset introduced in \cite{ramasesh2021anatomy}. In our experiment, both of the two tasks are 5-class classification problems, where each class is one of the 20 superclasses of the CIFAR-100 dataset. For instance, we take the five superclasses {\it aquatic mammals}, {\it fruits and vegetables}, {\it household electrical devices}, {\it trees}, and {\it vehicles-1}. The corresponding subclasses for Task 1 are (1) {\it dolphin}, (2) {\it apple}, (3) {\it lamp}, (4) {\it maple tree}, and (5) {\it bicycle}, while for Task 2, they are (1) {\it whale}, (2) {\it orange}, (3) {\it television}, (4) {\it willow}, and (5) {\it motorcycle}. Figure \ref{fig:augmentation} shows sample images and five random augmentations for the classes in both tasks.

\paragraph{Setup.}
In all experiments, we used a Wide-ResNet \citep{zagoruyko2016wide} as our backbone. The network has 16 layers, a widening factor of $ 4 $, and a dropout rate of $0.2$. We leveraged random flip, translation, and cutout \citep{devries2017improved} as augmentation.
We use ADAM as our optimizer for all experiments, with learning rate $10^{-3}$ and momentum $0.9$. The importance parameter $ \lambda $ for each algorithm is: $ 10^5 $ for EWC, $ 10^2 $ for Sketched EWC, $ 10^5 $ for MAS, and $ 10^3 $ for Sketched MAS. The minibatch size is $64$. The online learning parameter is $ \alpha = 0.25 $ for all experiments. In Sketched SR algorithms, we uses 50 sketches to approximate the full importance matrix. All reported results are averaged over $10$ runs with different random seeds.

\end{document}